\documentclass{article}

\PassOptionsToPackage{numbers,compress}{natbib}
\usepackage[preprint]{neurips_2026}

\makeatletter
\renewcommand{\@notice}{%
  \enlargethispage{2\baselineskip}%
  \@float{noticebox}[b]%
    \centering\normalsize
    DISTRIBUTION STATEMENT A: Approved for public release, distribution is unlimited.%
  \end@float%
}
\makeatother

\usepackage[utf8]{inputenc}
\usepackage[T1]{fontenc}
\usepackage{hyperref}
\usepackage{url}
\usepackage{booktabs}
\usepackage{amsfonts}
\usepackage{nicefrac}
\usepackage{microtype}
\usepackage{xcolor}
\usepackage{amsmath,amssymb,amsthm}
\usepackage{bm}
\usepackage{graphicx}
\usepackage{algorithm}
\usepackage{algpseudocode}
\usepackage{multirow}
\usepackage{enumerate}
\usepackage{subfig}
\usepackage{overpic}

\newtheorem{corollary}{Corollary}
\newtheorem{proposition}{Proposition}
\newtheorem{remark}{Remark}

\newcommand{\x}{\bm{x}}
\newcommand{\xk}{\bm{x}_k}
\newcommand{\uk}{\bm{u}_k}
\newcommand{\tauk}{\tau_k}
\newcommand{\taumin}{\tau_{\min}}
\newcommand{\taumax}{\tau_{\max}}
\newcommand{\V}{V}
\newcommand{\Kgain}{K}
\newcommand{\Acl}{A_{\mathrm{cl}}}
\newcommand{\Vscale}{V_{\mathrm{scale}}}
\newcommand{\MSI}{\mathrm{MSI}}
\newcommand{\RTA}{\mathrm{RTA}}
\DeclareMathOperator{\tr}{tr}
\DeclareMathOperator{\diag}{diag}
\newcommand{\R}{\mathbb{R}}

\title{Learning When to Act: Communication-Efficient\\
Reinforcement Learning via Run-Time Assurance}

\author{%
  Adam Haroon$^{1,3}$\thanks{Corresponding author: \texttt{aharoon@iastate.edu}.
  } \quad
  Erick J.\ Rodr\'iguez-Seda$^{2}$ \quad
  Cody Fleming$^{1}$ \quad
  Tristan Schuler$^{3}$\\[4pt]
  $^{1}$Department of Mechanical Engineering, Iowa State University, Ames, IA, USA\\
  $^{2}$Department of Weapons, Robotics, and Control Engineering,\\
         United States Naval Academy, Annapolis, MD, USA\\
  $^{3}$Navy Center for Applied Research in Artificial Intelligence (NCARAI),\\
         U.S.\ Naval Research Laboratory, Washington, D.C., USA
}

\begin{document}
\maketitle

\begin{abstract}
Safe reinforcement learning (RL) typically asks \emph{what} an agent should do.
We ask \emph{when} it needs to act, and show that a single policy can
jointly learn control inputs and communication-efficient timing decisions
under a pointwise Lyapunov safety shield.
We scope the framework to stabilization around a known equilibrium, where
CARE-based LQR backups, Lyapunov certificates, and classical Lyapunov-STC
are well defined, enabling a clean comparison against the analytical
baseline.
A run-time assurance (RTA) layer overrides the policy pointwise via a
one-step-ahead Lyapunov prediction and a precomputed LQR backup, providing
a strictly stronger guarantee than constrained MDP methods that enforce
safety only in expectation.
On an inverted pendulum, cart--pole, and planar quadrotor, the learned
policy achieves $1.91\times$, $1.45\times$, and $3.51\times$ higher mean
inter-sample interval (MSI) than a classical Lyapunov-triggered baseline;
a fixed LQR controller at the same average rate is unstable on all three
plants, showing that adaptive timing, not a lower average rate, is what
makes sparsity safe.
A CARE-derived Lyapunov reward transfers across environments without
redesign, with a single weight $w_c$ controlling the
stability--communication tradeoff; ablations confirm the RTA shield is
essential, with its removal reducing MSI by $1.27$--$1.84\times$ and
degrading state norms.
A preference-conditioned extension recovers the full tradeoff frontier
from a single model at $\tfrac{2}{11}$ of training compute, and SAC
experiments confirm the results are algorithm-agnostic across discrete
and continuous domains.
A 12-state 3D quadrotor case study extends the framework to
higher-dimensional systems where classical STC design is analytically
intractable: a SAC agent reaches MSI $= 0.302$\,s ($94\%$ of $\taumax$)
at $0\%$ RTA, while classical Lyapunov-STC remains pinned at $\taumin$
and a fixed-rate LQR controller at the same average interval crashes
within two control updates.
Robustness to $\pm30\%$ plant-mass variation and additive disturbances
confirms graceful degradation, with the RTA absorbing what the learned
policy cannot.
\end{abstract}

\section{Introduction}
\label{sec:intro}

Safe reinforcement learning (RL) has made substantial progress on
\emph{what} an agent should do, but has largely ignored the question of
\emph{when} it needs to act. Fixed-rate control executes at every
timestep regardless of whether the state has changed meaningfully; in
safety-critical cyber-physical systems, this is expensive, consuming
sensing, computation, and bandwidth on every update. Lifting this
assumption without sacrificing safety is the central problem we address.

The control theory literature studies this as self-triggered control (STC)
and event-triggered control
(ETC)~\cite{heemels2012introduction,mazo2011decentralized,anta2010sample,gommans2014self,wang2010event}:
classical analytical methods determine the next update time from a
Lyapunov or model-based triggering rule, providing strong stability
guarantees but becoming difficult to apply to nonlinear, underactuated,
or high-dimensional dynamics where the linearized one-step prediction
leaves significant communication savings on the table.
Learning-based approaches relax this conservatism by approaching the
admissibility boundary
empirically~\cite{baumann2018deep,wan2023model,funk2021learning,wang2021deep,aggarwal2025interq,sedghi2022machine,treven2024sense},
but existing formulations either omit formal safety guarantees on the
triggering decision~\cite{baumann2018deep,wan2023model}, rely on
model-based learning of the dynamics~\cite{funk2021learning}, or
address an interaction-cost continuous-time setting without a hard
timing-safety mechanism~\cite{treven2024sense}.
Our framing differs on multiple axes: we operate in discrete-time STC
where $\tauk$ is selected at each decision instant, with both
discrete-action (DQN) and continuous-action (SAC) agents; we are
model-free with a precomputed Linear-Quadratic Regulator (LQR) backup
controller; and we provide
\emph{pointwise} Lyapunov-decreasing safety via a hard RTA override, a
formal guarantee that expectation-level constraint methods such as
Lagrangian-relaxation~\cite{altman2021constrained,achiam2017constrained}
cannot provide by construction (Section~\ref{sec:lagrangian}).

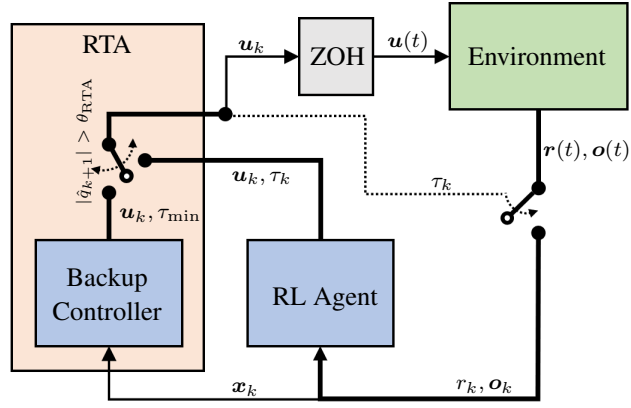
\begin{figure}[t]
  \centering
  \begin{overpic}[width=0.59\linewidth,grid=false]{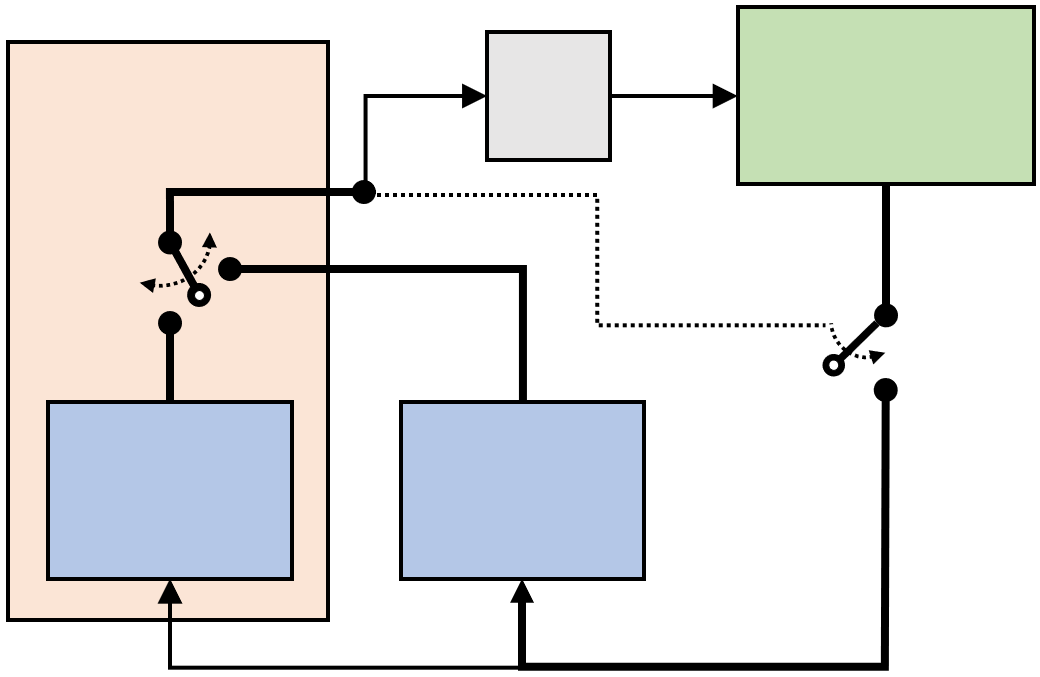}
    \put(16,20){\makebox(0,0){Backup}}
    \put(16,15){\makebox(0,0){Controller}}
    \put(51,17){\makebox(0,0){RL Agent}}
    \put(85,56){\makebox(0,0){Environment}}
    \put(53,56){\makebox(0,0){ZOH}}
    \put(16,58){\makebox(0,0){RTA}}
    \put(86,40){\footnotesize$\bm{r}(t),\bm{o}(t)$}
    \put(61,57.5){\footnotesize$\bm{u}(t)$}
    \put(72,2.5){\footnotesize$r_k,\bm{o}_k$}
    \put(36,2.0){\footnotesize$\xk$}
    \put(36,36){\footnotesize$\bm{u}_k,\tau_k$}
    \put(18,30){\footnotesize$\bm{u}_k,\tau_{\min}$}
    \put(37,57){\footnotesize$\bm{u}_k$}
    \put(68,35){\footnotesize$\tau_k$}
    \put(11,32){\rotatebox{90}{\tiny$|\hat{q}_{k+1}|>\theta_{\RTA}$}}
  \end{overpic}
  \caption{Proposed self-triggered control framework. The RL agent outputs
    a joint action $(\bm{u}_k,\tau_k)$. The RTA evaluates the
    one-step-ahead safety predicate $|\hat{q}_{k+1}| > \theta_{\RTA}$ on
    the linearized prediction of the next state and, if violated,
    substitutes the LQR backup at $\taumin$. The selected input is held
    by a zero-order hold (ZOH) until the next sampling instant
    $t_k+\tau_k$.}
  \label{fig:RL_STC_RTA_Scheme}
\end{figure}

We propose a unified RL framework (Fig.~\ref{fig:RL_STC_RTA_Scheme})
that jointly selects the control input and the next inter-sample
interval, supervised by an RTA layer acting as a safety
shield~\cite{hobbs2023runtime,lazarus2020runtime,seto1998simplex}. The
RTA uses a precomputed LQR backup and a one-step-ahead safety
prediction to override unsafe actions, with each backup intervention
provably unsaturated and
Lyapunov-decreasing~\cite{alshiekh2018safe,miller2024optimal,dunlap2023run}.
Episodes are time-bounded (not step-bounded), so a larger MSI directly
reduces the number of agent decisions per episode of fixed duration:
the operational quantity that communication efficiency targets. The
RTA differs fundamentally from constrained MDP
methods~\cite{altman2021constrained,achiam2017constrained,gu2024review,brunke2022safe,garcia2015comprehensive},
which enforce safety only in expectation; we demonstrate empirically that a
Lagrangian-DQN baseline using the identical constraint predicate
achieves $21$--$52\%$ lower MSI than RL-STC while accumulating up to
$35\%$ hard safety violations at the best checkpoint
(Section~\ref{sec:lagrangian}). The framework is scoped to
stabilization around a known equilibrium where CARE, Lyapunov analysis,
and classical STC are well defined, enabling clean comparison; we
evaluate on three lower-dimensional plants (inverted pendulum,
CartPole, planar quadrotor) plus a 12-state 3D quadrotor case study
(Section~\ref{sec:casestudy}) where analytical STC design is
intractable.

\paragraph{Main contributions.}
\textbf{(1)} A joint RL formulation that simultaneously learns control
inputs and inter-sample intervals under a hard pointwise safety
certificate (Proposition~\ref{prop:rta_backup}); unlike prior shielded
RL that enforces safety at fixed time steps, our shield acts on the
timing decision $\tauk$ itself. A single weight $w_c$ traces the full
stability/communication tradeoff frontier.
\textbf{(2)} A CARE-derived Lyapunov reward that transfers across SISO
and MIMO environments without redesign.
\textbf{(3)} Systematic empirical validation on three lower-dimensional
plants: $1.45$--$3.51\times$ MSI gains over Classical Lyapunov-STC,
ablations isolating the RTA's role, a Lagrangian-DQN comparison, a
preference-conditioned extension at $\tfrac{2}{11}$ compute, SAC
confirming algorithm-agnosticism, and robustness under $\pm30\%$ mass
variation and additive disturbances.
\textbf{(4)} A 12-state 3D quadrotor case study showing the framework
scales to higher-dimensional MIMO systems where analytical STC design
is intractable and where Proposition~\ref{prop:rta_backup}'s formal
certificate does not extend at the chosen $\taumin$, testing the
framework's empirical reach beyond where the certificate applies (MSI
$=0.302\,\mathrm{s}$, $94\%$ of $\taumax$, $0\%$ RTA; Classical STC
pinned at $\taumin$; B2 unstable within two control updates).

\section{Problem Formulation}
\label{sec:problem}

\subsection{Evaluation Scope}
\label{sec:eval_scope}

The framework targets stabilization around a known equilibrium. Three
structural conditions are required for the safety certificates developed
below: (i) a valid local linearization, (ii) a positive-definite quadratic
Lyapunov function, and (iii) accuracy of the one-step-ahead prediction
in~\eqref{eq:rta_pred}. Stabilization around a known equilibrium guarantees
all three: the linearization at the equilibrium is well posed, the CARE
solution yields $V(x) = x^\top P x$ that is positive-definite about that
point, and the prediction is accurate within its neighborhood.
Propositions~\ref{prop:rta_backup} and~\ref{prop:tau_star} therefore hold
under this scope.

\subsection{Self-Triggered Control}
\label{sec:stc}

Consider a continuous-time nonlinear plant
$\dot{\x}(t)=f(\x(t),\bm{u}(t))$,
where $\x\in\R^n$ is the state and $\bm{u}\in\R^m$ is a piecewise-constant
control input. In self-triggered control the controller executes at
sampling instants $\{t_0,t_1,\ldots\}$ determined online. At each instant
$t_k$ the controller observes $\xk\triangleq\x(t_k)$, selects $\uk$ held
constant over $[t_k,t_{k+1})$, and decides
\begin{equation}
  t_{k+1}=t_k+\tauk,\qquad
  \tauk\in\mathcal{T}=\{\taumin,2\taumin,\ldots,N\taumin\},
  \label{eq:trigger}
\end{equation}
where $N$ is the cardinality of $\mathcal{T}$; for consistency we
choose $N=8$ across all experiments. The mean sampling interval (MSI)
is the causal $n$-point moving average
$\MSI_k=[(n{-}1)\MSI_{k-1}+\tauk]/n$, initialized at
$\MSI_0=\taumin$; in what follows we take $n=5$. A larger MSI corresponds to sparser communication; the
objective is to maximize MSI subject to closed-loop stability.
Episodes are time-bounded by a fixed simulation horizon $T_{\max}$ rather
than a fixed step count, so a larger MSI directly reduces the number of
agent decisions per episode. This reflects what communication efficiency
means in practice: fewer control updates per unit of operating time, and
therefore lower demand on sensing, computation, and network bandwidth.

\subsection{Run-Time Assurance}
\label{sec:rta_formulation}

Run-time assurance augments the RL policy with a precomputed LQR backup
(Section~\ref{sec:lyapunov}). At each step the linearized one-step-ahead
prediction of the safety-critical scalar $q_k=\bm{c}^\top\xk$ is
\begin{equation}
  \hat{q}_{k+1}=q_k+\tauk\dot{q}_k+\tfrac{1}{2}\tauk^2\,
    \ddot{q}_{\mathrm{lin}}(\xk,\uk).
  \label{eq:rta_pred}
\end{equation}
If $|\hat{q}_{k+1}|>\theta_{\RTA}$ (or a position bound is exceeded), the
RL action is overridden:
\begin{equation}
  \uk\leftarrow\mathrm{clip}_{\mathrm{cw}}(-\Kgain\xk,-\bm{u}_{\max},+\bm{u}_{\max}),
  \quad\tauk\leftarrow\taumin.
  \label{eq:rta_override}
\end{equation}
The threshold $\theta_{\RTA}$ is set strictly below the LQR saturation
angle $\theta_{\mathrm{sat}}=u_{\max}/|K_\theta|$ so that the backup
retains full authority. The role of RL is not to guarantee safety, but to
maximize performance within the safe set defined by the RTA layer,
discovering non-conservative inter-sample intervals that would be difficult
to obtain analytically.

\begin{proposition}[ZOH Lyapunov Decrease for the Linearized Backup]
\label{prop:rta_backup}
Let
\begin{equation}
  M(\tau) \triangleq e^{A\tau}
   - \biggl(\!\int_0^\tau\!e^{As}\,ds\biggr) B\Kgain
  \label{eq:zoh_M}
\end{equation}
denote the ZOH-discretized closed-loop transition matrix that arises when
the backup command $\uk = -\Kgain\xk$ is held constant on
$[t_k,\,t_k+\tau)$, and suppose this command is component-wise
unsaturated at $\xk$ ($|[-\Kgain\xk]_i| < u_{\max,i}$ for all $i$).
Then $\x_{k+1} = M(\taumin)\,\xk$ on the linearized plant, and
$\V(\x_{k+1}) \leq \V(\xk)$ holds for all $\xk$ if and only if
\begin{equation*}
  M_{\mathrm{disc}} \triangleq P - M(\taumin)^\top P\,M(\taumin)
  \;\succeq\;0.
\end{equation*}
Numerical verification (Table~\ref{tab:theory}) confirms
$M_{\mathrm{disc}} \succ 0$ for the Pendulum, CartPole, and planar
Quadrotor environments. The Quadrotor3D plant fails this verification
at its chosen $\taumin$, retained deliberately for the case study
(Section~\ref{sec:casestudy}) which probes empirical behavior in the
regime where the certificate does not extend.
\end{proposition}
\emph{Proof and an extended scope discussion are in
Appendix~\ref{app:proofs}.}

\section{Method}
\label{sec:method}

\subsection{CARE-Based Lyapunov Function}
\label{sec:lyapunov}

For each environment we linearize the dynamics around the target equilibrium
to obtain $\dot{\x}=A\x+B\bm{u}$. The Lyapunov function
$\V(\x)=\x^\top P\x$ is obtained by solving the
Continuous-time Algebraic Riccati Equation (CARE)
\begin{equation}
  A^\top P+PA-PBR^{-1}B^\top P+Q=0,
  \label{eq:care}
\end{equation}
with $Q\succ0$, $R\succ0$. The optimal LQR gain is $\Kgain=R^{-1}B^\top P$
and the closed-loop matrix $\Acl=A-B\Kgain$ is Hurwitz. By the CARE
identity, $\dot{\V}(\x)\leq-\lambda\,\V(\x)$ along the linear closed-loop
trajectory, where
$\lambda=\lambda_{\min}(Q+\Kgain^\top R\Kgain,\,P)$
is the minimum generalized eigenvalue~\cite{boyd1994linear}, yielding a
tighter bound than the standard estimate. The Lyapunov value is normalized
by $\Vscale=\tr(P)/n$.

\begin{proposition}[Admissible Inter-Sample Interval]
\label{prop:tau_star}
For the linear closed-loop system with $P\succ0$ solving~\eqref{eq:care},
for any $\xk\neq\bm{0}$ there exists $\tau^*(\xk)>0$ such that
$\V(\x_{k+1})<\V(\xk)$ for all $\tauk\in(0,\tau^*(\xk))$, where
$\x_{k+1}=M(\tauk)\,\xk$ is the ZOH solution under $\uk=-\Kgain\xk$ with
$M(\cdot)$ from~\eqref{eq:zoh_M}.
On any compact annular region $c\leq\V(\xk)\leq c'$, a uniform constant
$\tau^+>0$ independent of $\xk$ exists with this property
(Corollary~\ref{cor:tau_plus}); $\taumin$ is verified numerically to
satisfy $\taumin\leq\tau^+$ via $M_{\mathrm{disc}}\succ0$
(Table~\ref{tab:theory}).
\end{proposition}

\emph{Proofs of Proposition~\ref{prop:tau_star} and
Corollary~\ref{cor:tau_plus} are in Appendix~\ref{app:proofs}.
Corollary~\ref{cor:nonlinear} (Appendix~\ref{app:proofs}) extends these
results to the nonlinear plant with an explicit stability radius $r^*$.}

\subsection{Action Space, Observations, and Reward}
\label{sec:reward_sec}

The agent selects a discrete action from a set of $(\tau,\bm{u})$ tuples.
For SISO environments $\mathcal{A}=\mathcal{T}\times\mathcal{U}$; for the
MIMO quadrotor
$\mathcal{A}=\mathcal{T}\times\mathcal{U}_{\delta F}\times\mathcal{U}_M$.
The agent observes $\bm{o}_k=[\xk^\top,\MSI_k,b_k]^\top$, where
$b_k\in\{0,1\}$ flags whether the RTA was active on the previous step;
including $\MSI_k$ enables credit assignment for the communication reward.

The per-step reward is
\begin{equation}
  r_k=\underbrace{r_{\mathrm{stab}}+1-\frac{\V(\x_{k+1})}{\Vscale}}_{%
      \text{stability}}
    +\underbrace{w_c\!\left(\frac{\MSI_k-\taumin}{\taumax-\taumin}
      \right)^{\!2}}_{\text{communication}}
    +\underbrace{100\,r_{\mathrm{safe}}}_{\text{safety}},
  \label{eq:reward}
\end{equation}
where $r_{\mathrm{stab}}=+1$ if
$\V(\x_{k+1})\leq\V(\xk)e^{-\lambda\tauk}$ (with a near-origin guard
$\V(\xk)<\frac{1}{4}\Vscale$ preventing penalization of residual
oscillations) else $-1$; the graded
$1-\V(\x_{k+1})/\Vscale$ term provides a signal proportional to the
absolute Lyapunov value; $r_{\mathrm{safe}}\in\{-1,0\}$ flags RTA
overrides; and a terminal penalty of $-1000$ applies on
constraint-violation termination. The weight $w_c$ controls the
stability/communication tradeoff; we sweep
$w_c\in\{0.25,0.5,1,2,4,6,8,10,12,14,16\}$.

\subsection{DQN Training and Preference-Conditioned Extension}
\label{sec:training}

A DQN ($256{\to}128{\to}128$ ReLU) is trained for $1\,\mathrm{M}$ steps
using Stable-Baselines3~\cite{SB3} with best-model checkpointing at the
per-step average-reward peak. Per-step rather than total-episode reward
is the natural metric since episodes are time-bounded: a longer MSI
yields fewer steps per fixed-duration episode. Per-step reward thus
decouples policy quality from $\tau_k$ choice.
To avoid 11 separate per-$w_c$ sweeps, a \emph{preference-conditioned
DQN}~\cite{abels2019dynamic,yang2019generalized} samples $w_c$ uniformly
at each episode reset from the same grid and appends a log-normalized
$\hat{w}_c\in[0,1]$ to the observation; at deployment $w_c$ is fixed
per-episode, so a single $2\,\mathrm{M}$-step model recovers the
per-$w_c$ frontier within $1$--$8\%$ at $\tfrac{2}{11}$ of total compute.

\section{Environments}
\label{sec:envs}

Four environments of increasing complexity span SISO and MIMO
underactuated dynamics: the inverted pendulum (2 states, 1 input;
Gymnasium Pendulum-v1), CartPole (4, 1; Gymnasium CartPole-v1
\cite{towers2024gymnasium}), the planar quadrotor (6, 2; coupled MIMO
hover with indirect horizontal control through tilt), and Quadrotor3D
(12, 4; 6-DOF rigid-body quadrotor) used as the higher-dimensional case
study. Each plant exhibits a distinct constraint geometry: the pendulum
has the narrowest safety margin ($\theta_{\mathrm{sat}}-\theta_{\RTA}=
1.9^\circ$); CartPole has a tight termination angle ($12^\circ$) and
large $\Vscale\approx56.6$ that weakens the graded stability term;
the planar quadrotor yields the largest MSI gain over Classical STC
($3.51\times$) despite a $9^\circ$ margin; Quadrotor3D's $5{,}000$-action
discrete space rules out DQN, motivating SAC on the continuous Box action
space, and is also deliberately chosen as a regime where
Proposition~\ref{prop:rta_backup}'s formal certificate does not extend
at the chosen $\taumin$, enabling us to probe the framework's empirical
reach beyond where the certificate applies
(Section~\ref{sec:casestudy}). All integrate at
$\Delta t = 0.001\,\mathrm{s}$. Full dynamics, parameters
(Table~\ref{tab:envparams}), and DQN/SAC hyperparameters are in
Appendices~\ref{app:envs} and \ref{app:hyperparams}; compute resources and
wall-clock times are in Appendix~\ref{app:compute}.

\section{Results}
\label{sec:results}

Each environment is trained for all 11 $w_c$ values; 100 deterministic
evaluation episodes are run per checkpoint. We report MSI (mean $\pm$ std),
RTA activation rate (RTA\%), and $L_2$ state norms
$P_i=\sqrt{\int_0^T s_i^2(t)\,dt}$. Full per-weight sweep tables for DQN,
SAC, and Preference-Conditioned DQN appear in
Appendix~\ref{app:full_tables}.

\subsection{DQN Results and Algorithm Comparison}
\label{sec:dqn_results}

Table~\ref{tab:headline} summarizes the best-model checkpoint MSI at the
best-checkpoint peak $w_c$ for each environment and algorithm. The DQN $w_c$ sweeps
and their corresponding ablation/Lagrangian comparisons are aggregated
across $3$ independent training seeds (Appendix~\ref{app:multi_seed});
all other rows use the standard single-seed protocol.
The pendulum saturates near $\taumax$ from $w_c=6$ onwards. CartPole
reaches $96\%$ of $\taumax$ at $w_c=16$ with $0.04\%$ RTA, despite a
non-monotonic final-model trend caused by the large $\Vscale$ saturating
the graded stability term. The quadrotor reaches $88\%$ of $\taumax$,
reflecting the harder coupled hover dynamics. SAC achieves $0\%$ RTA
across every weight and environment in both the final model and best
checkpoint, with no high-$w_c$ collapses, confirming that DQN collapses
are a property of the discrete-action optimizer rather than the
framework. The preference-conditioned DQN recovers the per-$w_c$
frontier within $1$--$8\%$ from a single $2\,\mathrm{M}$-step model,
confirming that the full tradeoff frontier can be recovered at a
fraction of total training compute.

\paragraph{Off-policy methods are required in practice.}
\label{sec:offpolicy_required}
PPO~\cite{schulman2017ppo} on the Pendulum either collapses to $\taumin$
or thrashes against the RTA shield at every $w_c$ (best MSI $\approx
0.09\,\mathrm{s}$ vs.\ DQN's $0.397\pm0.001\,\mathrm{s}$); the
rare-transition $-100$ RTA and $-1000$ terminal penalties are inherently
easier for replay-based off-policy methods to leverage. The framework
is algorithm-agnostic in principle, but practical convergence within
standard training budgets favors off-policy methods (full sweep:
Table~\ref{tab:ppo_pendulum}, Appendix~\ref{app:full_tables}).

\begin{table}[t]
\caption{Best-model checkpoint MSI (s) at the best-checkpoint peak $w_c$ per algorithm.
  DQN values are 3-seed mean $\pm$ std (Appendix~\ref{app:multi_seed});
  SAC and Pref-DQN values are single-seed (seed~0). All three algorithms
  substantially exceed the Classical STC baseline and a fixed LQR at the
  same average rate is unstable on all three plants (B2,
  Table~\ref{tab:baselines_all}). Full per-weight sweep tables are in
  Appendix~\ref{app:full_tables}.}
\label{tab:headline}
\centering
\small
\begin{tabular}{lcccccc}
\toprule
 & \multicolumn{2}{c}{Pendulum} & \multicolumn{2}{c}{CartPole}
 & \multicolumn{2}{c}{Quadrotor} \\
\cmidrule(lr){2-3}\cmidrule(lr){4-5}\cmidrule(lr){6-7}
Algorithm & MSI (s) & $w_c$ & MSI (s) & $w_c$ & MSI (s) & $w_c$ \\
\midrule
DQN                   & $0.397\pm0.001$ & 10 & $0.308\pm0.014$ & 16 & $0.281\pm0.018$ & 16 \\
SAC                   & 0.379 & 10 & 0.258 & 16 & 0.312 & 16 \\
Pref-DQN (1 model)    & 0.393 & 16 & 0.316 & 14 & 0.266 & 16 \\
\midrule
Classical STC (B3)    & \multicolumn{2}{c}{0.202} & \multicolumn{2}{c}{0.212} & \multicolumn{2}{c}{0.080} \\
\bottomrule
\end{tabular}
\end{table}

\subsection{Cross-Environment Analysis}
\label{sec:cross_env}

Figure~\ref{fig:msi_vs_wc} summarizes the best-model MSI vs.\ $w_c$;
training dynamics (MSI, RTA activation, and episode reward curves) are in
Appendix~\ref{app:training} (Fig.~\ref{fig:training_curves}).
Three patterns emerge across DQN, SAC, and Preference-Conditioned DQN.
\textbf{(1)} Higher $w_c$ accelerates MSI exploration: low-$w_c$
policies converge to $\MSI\approx\taumin$; high $w_c$ pushes $\tauk$
toward $\tau^*(\xk)$ (Proposition~\ref{prop:tau_star}), and all three
methods exceed Classical Lyapunov-STC at moderate-to-high $w_c$.
\textbf{(2)} At high $w_c$ the optimizer overshoots admissibility,
triggering sustained RTA intervention and reward degradation; the
best-model checkpoint captures the MSI peak before this onset, with
the canonical ablation $w_c$ rising with complexity: $w_c=8$ for
Pendulum (anchored within the saturation plateau, $\geq96\%$ of
$\taumax$) and $w_c=16$ for CartPole and Quadrotor (strict
best-checkpoint peak).
\textbf{(3)} Final-model degradation depends on plant geometry: Pendulum
and Quadrotor tolerate high $w_c$, while CartPole's large
$\Vscale\approx56.6$ and $12^\circ$ termination margin produce
non-monotonic final-model MSI; the best-model checkpoint is especially
critical there, yielding a clean tradeoff
($0.308\pm0.014\,\mathrm{s}$, $96\%$ of $\taumax$).
Detailed per-environment analyses are in Appendix~\ref{app:per_env}.

\begin{figure}[t]
  \centering
  \includegraphics[width=\textwidth]{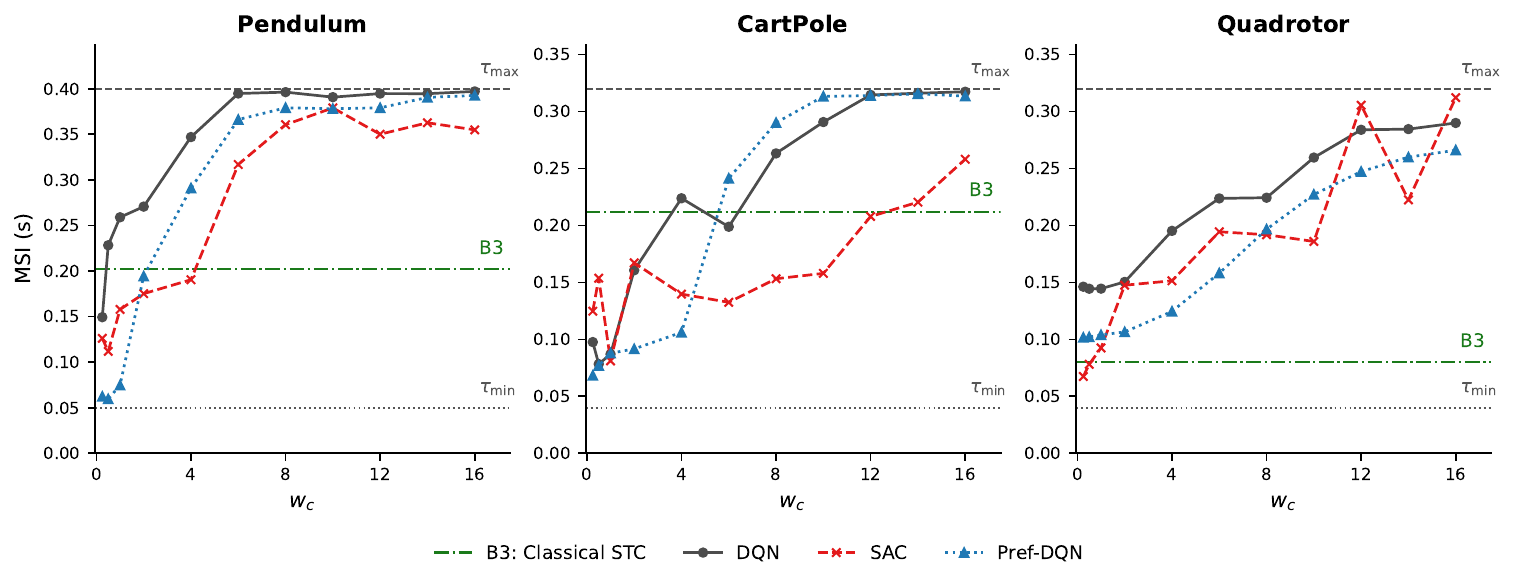}
  \caption{Best-model MSI vs.\ $w_c$ for DQN, SAC, and Preference-Conditioned
    DQN across all three environments (seed-0 curves; multi-seed DQN
    aggregates are in Tables~\ref{tab:pend}--\ref{tab:quad}). Reference
    lines mark $\taumin$ (dotted), $\taumax$ (dashed), and Classical
    Lyapunov-STC baseline B3 (dash-dot green). All algorithms exceed B3
    at moderate-to-high $w_c$; the gap above B3 represents
    communication efficiency that conservative analytical methods leave
    on the table. DQN saturates near $\taumax$ for Pendulum and
    CartPole; SAC converges more reliably but peaks slightly lower;
    Pref-DQN matches DQN within $1$--$8\%$ from a single model.}
  \label{fig:msi_vs_wc}
\end{figure}

\subsection{Baseline Comparisons}
\label{sec:baselines}

Three baselines isolate the sources of communication efficiency
($N_{\mathrm{eval}}=100$ each).

\textbf{B1 -- Fixed LQR at $\taumin$:} Always stable, never
communication-efficient. Establishes the hard safety floor.

\textbf{B2 -- Fixed LQR at $\tau_{\mathrm{match}}$:} The same LQR
controller at a fixed interval equal to the seed-0 best-checkpoint MSI:
$0.397\,\mathrm{s}$ (Pendulum), $0.317\,\mathrm{s}$ (CartPole),
$0.290\,\mathrm{s}$ (Quadrotor). \emph{This baseline is unstable on all
three plants} (mean episode lengths $1.1$--$2.4\,\mathrm{s}$), proving
that adaptive timing, not merely a reduced average rate, is what makes
sparsity safe: each plant's unstable mode has a time constant shorter
than $\tau_{\mathrm{match}}$ (e.g., $\sqrt{2l/(3g)}\approx0.258\,\mathrm{s}$
for the Pendulum's rod dynamics), so the fixed-rate ZOH closed loop diverges. The RL
policy avoids this by sampling fast near the unstable manifold and
stretching $\tauk$ near equilibrium (Proposition~\ref{prop:tau_star}).

\textbf{B3 -- Classical Lyapunov-STC:} At each step the LQR control is
applied and $\tauk$ is chosen greedily as the largest value in
$\mathcal{T}$ satisfying $\V(\tilde{\x}(\tauk))\leq\V(\xk)e^{-\lambda\tauk}$.
This baseline maintains stability but achieves $1.91\times$, $1.45\times$,
and $3.51\times$ lower MSI than RL-STC on Pendulum, CartPole, and
Quadrotor respectively (Table~\ref{tab:baselines_all}); the gap is largest for
the Quadrotor as the tight $\dot{\theta}$ coupling pins the trigger
near $\taumin$. The RL policy discovers that longer intervals
are safe on average even where the instantaneous linearized Lyapunov
bound is violated in the prediction.

\begin{table}[t]
\caption{Baseline comparison ($N_{\mathrm{eval}}=100$). RL-STC values
  are seed~0 to align with the $\tau_{\mathrm{match}}$ used for B2
  (the seed-0 best-checkpoint MSI; multi-seed canonical bests differ
  by less than the $\tau$-grid spacing per
  Appendix~\ref{app:ablation_tau}).
  \textit{Italics} = system failure (ep.\ length ${<}3\,\mathrm{s}$);
  ``--'' = norms not meaningful for failed episodes.}
\label{tab:baselines_all}
\centering
\small
\resizebox{\textwidth}{!}{%
\begin{tabular}{lccccccccccccc}
\toprule
& \multicolumn{3}{c}{Pendulum}
& \multicolumn{5}{c}{CartPole}
& \multicolumn{5}{c}{Quadrotor} \\
\cmidrule(lr){2-4}\cmidrule(lr){5-9}\cmidrule(lr){10-14}
Method
  & MSI & $P_3(\theta)$ & $P_4(\dot\theta)$
  & MSI & $P_1(x)$ & $P_2(\dot{x})$ & $P_3(\theta)$ & $P_4(\dot\theta)$
  & MSI & $P_1(x)$ & $P_2(\dot{x})$ & $P_3(\theta)$ & $P_4(\dot\theta)$ \\
\midrule
B1: LQR @ $\taumin$
  & 0.050 & 0.028 & 0.092
  & 0.040 & 0.069 & 0.078 & 0.014 & 0.048
  & 0.040 & 0.166 & 0.166 & 0.033 & 0.120 \\
B2: LQR @ $\tau_{\mathrm{match}}$
  & \textit{0.397} & \textit{--} & \textit{--}
  & \textit{0.317} & \textit{--} & \textit{--} & \textit{--} & \textit{--}
  & \textit{0.290} & \textit{--} & \textit{--} & \textit{--} & \textit{--} \\
B3: Classical STC
  & 0.202 & 0.026 & 0.167
  & 0.212 & 0.065 & 0.094 & 0.013 & 0.097
  & 0.080 & 0.160 & 0.164 & 0.034 & 0.349 \\
\textbf{RL-STC}
  & \textbf{0.397} & 0.605 & 0.916
  & \textbf{0.317} & 1.637 & 2.361 & 0.325 & 2.105
  & \textbf{0.290} & 2.640 & 1.916 & 0.463 & 5.664 \\
\bottomrule
\end{tabular}}
\end{table}

The RL-STC policy simultaneously achieves the high inter-sample sparsity of
B2 \emph{and} the full-episode stability of B1, demonstrating that the
joint learning of control input and timing is necessary to unlock the
communication efficiency that neither fixed-rate nor greedy trigger-based
controllers can deliver. All reported $P_i$ values correspond to episodes
that completed the full $50\,\mathrm{s}$ horizon, confirming that larger
excursions reflect higher-MSI trajectories rather than unstable operation.

\subsection{Ablation Study}
\label{sec:ablation}

\paragraph{Ablation A -- Removing the RTA Shield.}
With the RTA override and penalty disabled ($w_c=8/16/16$,
canonical per Section~\ref{sec:cross_env}), best-checkpoint MSI drops
by $1.27$--$1.84\times$ (Pendulum: $0.385\to0.266\,\mathrm{s}$;
CartPole: $0.308\to0.167\,\mathrm{s}$;
Quadrotor: $0.281\to0.222\,\mathrm{s}$, Table~\ref{tab:ablation_rta}),
and state norms degrade substantially on CartPole and Quadrotor
($P_1(x)$: $2.95\to5.81$ on CartPole, $2.33\to6.99$ on Quadrotor;
$P_3(\theta)$ on Quadrotor degrades $2.30\times$). The Lyapunov reward
alone cannot replicate the RTA's role: the shield shapes the feasible
policy space during training and provides a hard safety floor at
deployment that the learned policy can leverage.

\begin{table}[t]
\caption{Ablation A -- No RTA vs.\ full method
  ($N_{\mathrm{eval}}=100$, best-model checkpoint, 3 seeds per row;
  MSI as mean $\pm$ std across the 3 per-seed mean values with
  $\mathrm{ddof}=1$, $P_i$ norms as mean only;
  $w_c=8/16/16$ for Pendulum/CartPole/Quadrotor).}
\label{tab:ablation_rta}
\centering
\small
\renewcommand{\arraystretch}{1.05}
\resizebox{\textwidth}{!}{%
\begin{tabular}{l | ccc | cccc | cccc}
\toprule
& \multicolumn{3}{c|}{Pendulum}
& \multicolumn{4}{c|}{CartPole}
& \multicolumn{4}{c}{Quadrotor} \\
\cmidrule(lr){2-4}\cmidrule(lr){5-8}\cmidrule(lr){9-12}
Method
  & MSI (s) & $P_3$ & $P_4$
  & MSI (s) & $P_1$ & $P_3$ & $P_4$
  & MSI (s) & $P_1$ & $P_3$ & $P_4$ \\
\midrule
No RTA
  & $0.266\pm0.041$ & $0.687$ & $2.432$
  & $0.167\pm0.042$ & $5.813$ & $0.428$ & $5.036$
  & $0.222\pm0.023$ & $6.993$ & $1.201$ & $7.862$ \\
\textbf{RL-STC}
  & $\bm{0.385\pm0.014}$ & $\bm{0.636}$ & $\bm{1.049}$
  & $\bm{0.308\pm0.014}$ & $\bm{2.955}$ & $\bm{0.401}$ & $\bm{2.796}$
  & $\bm{0.281\pm0.018}$ & $\bm{2.334}$ & $\bm{0.523}$ & $\bm{5.171}$ \\
\bottomrule
\end{tabular}}
\end{table}

Figure~\ref{fig:tsne_comparison} provides geometric evidence: the
RTA-enabled policy (top) occupies a compact high-reward region with
$1.27$--$1.84\times$ fewer control steps per episode (matching the
inverse MSI ratios), while the no-RTA policy (bottom) visits a much
wider region at substantially lower per-step rewards, isolating the
RTA as a geometric constraint that enables long-interval strategies
the Lyapunov reward alone cannot enforce.

\begin{figure}[t]
  \centering
  \includegraphics[width=\textwidth]{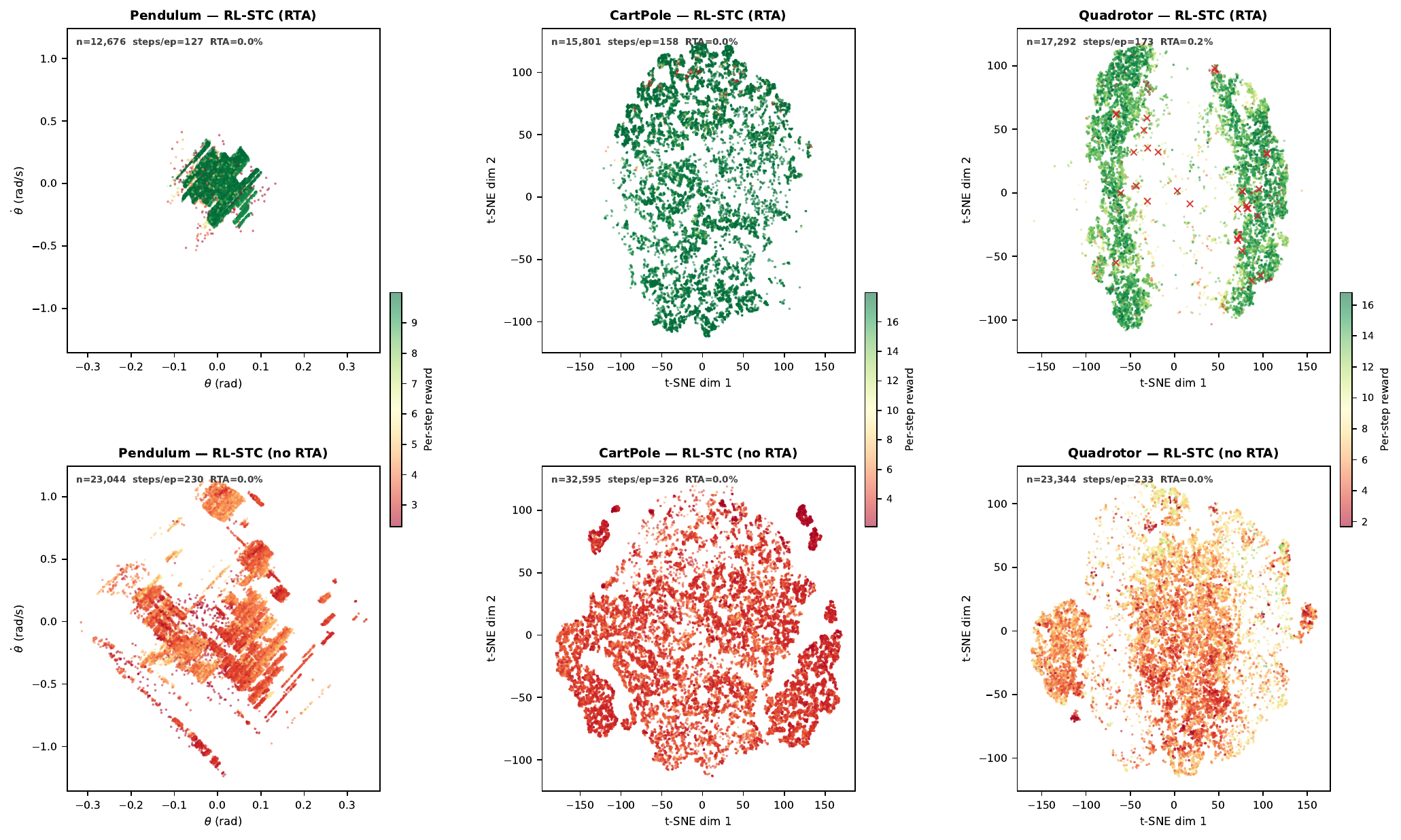}
  \caption{Per-step reward distribution (100 episodes; seed-0
    trajectories per Appendix~\ref{app:multi_seed}): RL-STC (top) vs.\
    no-RTA (bottom). Pendulum axes are $(\theta,\dot\theta)$;
    CartPole and Quadrotor use a shared t-SNE embedding fitted jointly
    on both policies' trajectories. Color encodes per-step reward
    (green high, red low). The RTA-enabled policy occupies a compact
    high-reward region with $1.27$--$1.84\times$ fewer steps per
    episode; the no-RTA policy visits a wider region at lower rewards.}
  \label{fig:tsne_comparison}
\end{figure}

\paragraph{Ablation B -- Fixed $\tau$.}
Fixing $\tau$ at the nearest grid value to $\tau_{\mathrm{match}}$ (with
RTA active) and learning only the control input degenerates to LQR at
$\taumin$ on Pendulum (RTA fires on $80.6\%$ of steps), reaches a
near-tie on CartPole only because $\tau_{\mathrm{match}}=\taumax$ there,
and gives up a $4.3\%$ MSI advantage on Quadrotor. Joint optimization of
$(\uk,\tauk)$ is not separable; full results in
Appendix~\ref{app:ablation_tau}.

\subsection{Comparison with Lagrangian Safe RL}
\label{sec:lagrangian}

A Lagrangian-DQN baseline replaces the hard RTA override with a soft
penalty $r_k^{\mathrm{Lag}}=r_k-\lambda_t c_k$ on the same one-step-ahead
predicate ($\lambda_t$ updated via dual gradient ascent). Across the
three lower-dimensional environments, Lagrangian-DQN achieves
$21$--$52\%$ lower best-checkpoint MSI than RL-STC and accumulates up
to $35.40\%$ hard violations on Quadrotor, while RL-STC has zero hard
violations by construction; final-model Lagrangian deteriorates further
(e.g., $43.49\pm48.55\%$ hard violations on Pendulum, $59.35\%$ on
Quadrotor) as $\lambda$ fails to hold the constraint across seeds. The
hard override decouples safety from inter-sample interval selection
during training, enabling long-interval commitment without the hedging
that a soft penalty forces. Full per-environment results are in
Appendix~\ref{app:lagrangian_full}, Table~\ref{tab:lagrangian}.

\subsection{Robustness}
\label{sec:robustness}

Best-model policies are evaluated under $\pm30\%$ mass mismatch and under
additive constant, periodic, and impulse disturbances injected directly
into the plant ($K$, $P$, and RTA thresholds fixed at nominal).
\emph{Classical Lyapunov-STC fails on Pendulum and CartPole at
$0.7\times$ mass}, the faster actual dynamics exceeding the nominal
prediction's stability limit; RL-STC remains stable across all scales
by trading MSI for safety. At $0.7\times$ CartPole mass MSI drops $37\%$
to $0.201\,\mathrm{s}$ with $13.91\%$ RTA activation and recovers fully
at $1.3\times$; domain randomization over $\pm40\%$ mass eliminates the
sensitivity at negligible cost to nominal performance.
Under disturbances, RTA activation grows monotonically with amplitude
while every episode completes safely (the filter absorbs what the policy
cannot), and RL-STC maintains higher MSI than Classical STC under all
conditions; the Quadrotor is structurally decoupled from thrust-channel
disturbances, with MSI varying by at most $0.002\,\mathrm{s}$ across all
conditions. Full mass-mismatch and disturbance tables
(Appendix~\ref{app:robustness}, Tables~\ref{tab:mismatch}
and~\ref{tab:disturbance}) and per-channel analysis
(Appendix~\ref{app:robustness_detail}) are in the supplementary
material.

\section{Case Study: Scaling to Higher-Dimensional Systems}
\label{sec:casestudy}

The three lower-dimensional environments share discrete action spaces
of $168$--$792$ actions that DQN can adequately cover. Quadrotor3D is
qualitatively different: its $5{,}000$-action discrete space leaves DQN
with under $0.5$ samples per action, rendering discrete learning
intractable. We adopt SAC~\cite{haarnoja2018sac} on the equivalent
continuous Box action space, keeping the reward, RTA, and Lyapunov
construction identical. Training extends to $2\,\mathrm{M}$ steps with
the policy widened to $256{\to}128{\to}128$, and the sweep extends to
$w_c\in\{0.25,1,4,8,16,24,32,40,48,56,64\}$.

Pointwise safety in our framework decomposes into a
\emph{deployment-time} override and a \emph{training-time} shaping
signal (the RTA penalty cascading to the terminal penalty when the
held backup propagates state out of bounds). The case study
deliberately probes the regime where the deployment certificate does
not extend: at $\taumin=0.04\,\mathrm{s}$ the ZOH-discretized closed
loop has spectral radius $1.19$ ($\tau_c\approx0.037\,\mathrm{s}$),
mirroring the conditions in which analytical STC design becomes
intractable on higher-dimensional systems. The certificate could be
restored locally by a smaller $\taumin$ giving $\tau_c>\taumin$, a
discrete-LQR redesign matched to $\taumin$, or a non-quadratic
Lyapunov certificate; we instead test whether the training-time
shaping alone produces a policy that respects the safety predicate.
The trained policy maintains $0\%$ RTA across the nominal sweep,
$\pm30\%$ mass perturbation, and the disturbance suite reported below,
a property of the trained policy on the tested distribution rather
than a formal guarantee for states or conditions outside it.

\paragraph{SAC sweep and Pref-SAC.}
The sweep reveals a two-phase dynamic: $w_c\le16$ keeps both final and
best near $\taumin$; $w_c\ge40$ saturates near $\taumax=0.32\,\mathrm{s}$
at $0\%$ RTA. We select $w_c=48$, where the best-model checkpoint
achieves MSI $=0.302\,\mathrm{s}$ ($94.3\%$ of $\taumax$, $0\%$ RTA)
with low attitude norms ($P_3(\varphi)\approx0.04$--$0.05$). A single
preference-conditioned SAC model ($4\,\mathrm{M}$ steps, $w_c$ sampled
per episode) reaches peak MSI $0.239\,\mathrm{s}$ ($79\%$ of the
per-$w_c$ SAC result) at $\le0.01\%$ RTA across the full grid,
recovering the frontier from one model.

\paragraph{Baselines, ablations, and Lagrangian-SAC.}
Classical Lyapunov-STC remains pinned at $\taumin$ throughout: the
one-step-ahead trigger is maximally conservative on the 12-state system,
firing at every step. Baseline 2 (LQR at $\tau_{\mathrm{match}}=
0.302\,\mathrm{s}$) crashes in $0.39\pm0.14\,\mathrm{s}$, fewer than two
steps, the most dramatic confirmation that adaptive timing is a
prerequisite. RL-STC achieves $7.5\times$ longer inter-sample intervals
than B1/B3 with attitude-rate norm $P_4(p)=0.190$ vs.\ $7.23$. Removing
the RTA (Ablation A) maintains MSI $\approx 0.310\,\mathrm{s}$ but
roughly doubles attitude norms ($P_3(\varphi)$: $0.044\to0.085$).
Lagrangian-SAC converges to MSI $\approx0.311$--$0.312\,\mathrm{s}$ with
zero hard violations, matching RL-STC empirically; SAC's entropy
regularization keeps the policy away from the constraint. Neither method
holds a pointwise certificate on Q3D, but RL-STC retains a constructive
path to Proposition~\ref{prop:rta_backup} via discrete-LQR redesign or a
smaller $\taumin$ (Section~\ref{sec:limitations}), while Lagrangian-SAC
admits no analytical extension.

\paragraph{Robustness.}
The nominal RL-STC policy shows \emph{zero} degradation under $\pm30\%$
mass perturbation (MSI and RTA unchanged at $0.302\,\mathrm{s}$ and
$0\%$), while Classical STC remains pinned at $\taumin$ throughout. It
also shows complete insensitivity to additive thrust disturbances
(constant up to $1.0\,\mathrm{N}$, periodic up to $1.5\,\mathrm{N}$ at
$2\,\mathrm{Hz}$, impulses up to $2.0\,\mathrm{N}$): MSI remains
$0.302\,\mathrm{s}$ with $0\%$ RTA across all seven conditions,
mirroring the structural decoupling in the 2D Quadrotor (the angular
RTA trigger monitors tilt while thrust disturbances perturb the
altitude channel). Full Q3D sweep, ablation, Lagrangian, mismatch, and
disturbance tables are in Appendix~\ref{app:q3d_detail}.

\section{Limitations and Future Work}
\label{sec:limitations}

The RTA's linearized one-step-ahead prediction grows conservative far
from the operating point, the direct mechanism behind the empirical MSI
ceiling and high-$w_c$ DQN collapses; a learned residual dynamics term
could raise this ceiling without sacrificing the hard safety floor. The
framework is scoped to stabilization around a known equilibrium so the
CARE solution, quadratic $V$, and linearized one-step prediction remain
well posed; non-equilibrium tasks (tracking, navigation, manipulation)
and purely black-box dynamics lie outside this scope. Control barrier
functions~\cite{gu2024review} offer a natural extension but require a
different backup construction. Scaling beyond the 12-state Quadrotor3D
to higher-dimensional systems, contact-rich dynamics, and real-world
platforms remains the primary direction for future work.

\section{Conclusion}
\label{sec:conclusion}

We have presented a unified RL framework for communication-efficient
self-triggered control with run-time assurance, evaluated across four
environments from a 2-state pendulum to a 12-state 3D quadrotor case
study. The CARE-derived Lyapunov reward transfers without redesign, and
Proposition~\ref{prop:tau_star} provides a formal basis for the observed
MSI ceiling. Three findings establish the framework's value:
\textbf{(i)} classical baselines fail (B2 unstable on all three
lower-dimensional plants and crashes within two control updates on Quadrotor3D;
B3 leaves a $1.45\text{--}3.51\times$ MSI gap and remains pinned at
$\taumin$ on Quadrotor3D);
\textbf{(ii)} on plants where $M_{\mathrm{disc}}\succeq0$, the RTA hard
override delivers Lyapunov-decrease safety by construction regardless
of algorithm, in contrast to Lagrangian-DQN which produces
$21\text{--}52\%$ lower MSI with up to $35\%$ hard violations;
\textbf{(iii)} ablations confirm RTA removal degrades MSI by
$1.27\text{--}1.84\times$ and that joint optimization of $(\uk,\tauk)$
is not separable. SAC and a preference-conditioned extension confirm
algorithm-agnostic gains and frontier recovery at $\tfrac{2}{11}$ of
total compute. The Quadrotor3D case study reaches MSI $=
0.302\,\mathrm{s}$ ($94\%$ of $\taumax$) with zero RTA activation,
extending the framework to systems where analytical STC design is
intractable and where the formal certificate does not extend at the
chosen $\taumin$, and robustness experiments across all four plants
confirm graceful degradation while every episode completes safely. The question
of \emph{when} to act is as consequential as \emph{what} to do, and
answering it safely and efficiently is tractable via the joint
optimization this framework provides.

\begin{ack}
A.\ Haroon conducted part of this work as an NREIP intern at the U.S.\
Naval Research Laboratory. The views expressed are those of the authors and
do not reflect the official policy or position of the U.S.\ Naval Academy,
Department of the Navy, Department of War, or U.S.\ Government.
\end{ack}

\bibliographystyle{abbrvnat}
\bibliography{references}

\newpage
\appendix

\section{Proofs}
\label{app:proofs}

\subsection{Proof of Proposition~\ref{prop:rta_backup}}

Under ZOH with $\uk = -\Kgain\xk$ held constant on
$[t_k,\,t_k+\taumin)$, the linear plant $\dot{\x} = A\x + B\bm{u}$
integrates to
\[
  \x(t_k+\taumin) = e^{A\taumin}\xk
                  + \biggl(\!\int_0^{\taumin}\!e^{As}\,ds\biggr) B\,\uk
                  = M(\taumin)\,\xk,
\]
so $\x_{k+1} = M(\taumin)\xk$. Substituting into $\V(\x) = \x^\top P\x$
gives
$\V(\x_{k+1}) = \xk^\top M(\taumin)^\top P\,M(\taumin)\,\xk$, so
$\V(\x_{k+1}) \leq \V(\xk)$ for all $\xk$ is equivalent to
$M_{\mathrm{disc}} \succeq 0$. Table~\ref{tab:theory} reports
$\lambda_{\min}(M_{\mathrm{disc}}) > 0$ for the three lower-dimensional
environments, directly verifying this condition.\hfill$\square$

\begin{remark}[Scope of Proposition~\ref{prop:rta_backup}]
Proposition~\ref{prop:rta_backup} certifies Lyapunov decrease in the
$P$-weighted norm under the \emph{linearized} ZOH closed loop with the
held LQR backup, on the lower-dimensional environments where
$M_{\mathrm{disc}} \succ 0$ holds. It is the discrete-time analog of
$\Acl$ being Hurwitz: each backup engagement weakly reduces the
CARE-based Lyapunov certificate, driving the linearized closed loop
toward the equilibrium. The proposition does not directly imply
component-wise satisfaction of $|q_{k+1}| < \theta_{\RTA}$, since decrease
in the quadratic form $V$ permits transient growth of individual state
components. It also does not, on its own, certify Lyapunov decrease on
the nonlinear plant under ZOH; Corollary~\ref{cor:nonlinear} provides a
complementary nonlinear bound under continuous feedback within the
conservative ball $\mathcal{B}(r^*)$. The choice
$\theta_{\RTA} < \theta_{\mathrm{sat}} = u_{\max}/|K_\theta|$ ensures
that, at engagement, the angle contribution to the backup satisfies
$|K_\theta q_k| < u_{\max}$; together with the unsaturation hypothesis it
provides the linear-ZOH grounding for the empirical pointwise safety
observed across all evaluation conditions in Section~\ref{sec:results}.
For Quadrotor3D, $M_{\mathrm{disc}} \succeq 0$ does not hold at the
chosen $\taumin$, so this deployment-time certificate does not extend;
the override remains active but uncertified, and the case study
(Section~\ref{sec:casestudy}) probes whether the training-time penalty
cascade alone produces a policy that respects the safety predicate.
\end{remark}

\subsection{Proof of Proposition~\ref{prop:tau_star}}

Under ZOH with $\uk=-\Kgain\xk$ held constant on $[t_k,\,t_k+\tau)$,
the linear closed-loop solution is $\x(\tau)=M(\tau)\,\xk$ with
$M(\cdot)$ from~\eqref{eq:zoh_M}. Define the Lyapunov difference
$\Delta\V(\tau)\triangleq\V(M(\tau)\xk)-\V(\xk)$; $M(0)=I$ gives
$\Delta\V(0)=0$. Differentiating,
\[
  \frac{d}{d\tau}\Delta\V(\tau)
  = \xk^\top\!\bigl(\dot{M}(\tau)^\top P\,M(\tau) + M(\tau)^\top P\,\dot{M}(\tau)\bigr)\,\xk.
\]
Differentiating~\eqref{eq:zoh_M} gives
$\dot{M}(\tau)=Ae^{A\tau}-e^{A\tau}B\Kgain=e^{A\tau}(A-B\Kgain)
=e^{A\tau}\Acl$, so $\dot{M}(0)=\Acl$. Substituting at
$\tau=0$ and applying the CARE identity
$\Acl^\top P+P\Acl=-(Q+\Kgain^\top R\Kgain)\prec0$,
\[
  \left.\frac{d}{d\tau}\Delta\V(\tau)\right|_{\tau=0}
  = \xk^\top(\Acl^\top P + P\Acl)\,\xk
  = -\xk^\top(Q+\Kgain^\top R\Kgain)\,\xk < 0
\]
for all $\xk\neq\bm{0}$. By continuity of $\Delta\V(\tau)$ in $\tau$,
there exists $\tau^*(\xk)>0$ such that $\Delta\V(\tau)<0$ for all
$\tau\in(0,\tau^*(\xk))$.\hfill$\square$

\subsection{Corollary~\ref{cor:tau_plus}: State-Independent Admissible Interval}
\label{cor:tau_plus}

\begin{corollary}
For any $0<c\leq c'$, let
$\mathcal{A}_{c,c'}=\{\xk:c\leq\V(\xk)\leq c'\}$. There exists a
constant $\tau^+>0$, independent of $\xk$, such that $\Delta\V(\tau)<0$
for all $\tau\in(0,\tau^+)$ and all $\xk\in\mathcal{A}_{c,c'}$. In
particular, $\taumin$ serves as such a $\tau^+$, verified numerically via
$M_{\mathrm{disc}}\succ0$ (Table~\ref{tab:theory}).
\end{corollary}

\begin{proof}
Let $M_Q\triangleq Q+\Kgain^\top R\Kgain\succ 0$ (distinct from
$M(\tau)$ in Eq.~\eqref{eq:zoh_M}); from
Proposition~\ref{prop:tau_star}'s derivative computation,
$d\Delta\V/d\tau|_{\tau=0}=-\xk^\top M_Q\xk$.
On $c\leq\V(\xk)\leq c'$, $\|\xk\|^2\geq c/\lambda_{\max}(P)$, so
$\xk^\top M_Q\xk\geq\lambda_{\min}(M_Q)c/\lambda_{\max}(P)>0$ uniformly.
Since $\Delta\V(0)=0$ and
$d\Delta\V/d\tau|_{\tau=0}\leq-\lambda_{\min}(M_Q)c/\lambda_{\max}(P)<0$
uniformly, joint continuity of $\Delta\V(\tau,\xk)$ implies the existence
of a uniform $\tau^+>0$.
\end{proof}

\begin{remark}
Under \emph{continuous} state feedback, $\Acl$ Hurwitz implies
$\V(e^{\Acl\tau}\xk)\leq\V(\xk)e^{-\lambda\tau}$ for all $\tau\geq0$~\cite{khalil2002nonlinear},
so $\tau^*(\xk)=\infty$. Under ZOH on the \emph{nonlinear} plant, the
true trajectory deviates from the linearized prediction $M(\tau)\xk$ as
$\tauk$ grows, eventually violating the Lyapunov decrease guarantee.
B2 confirms this: the fixed-rate ZOH controller at
$\tau_{\mathrm{match}}$ is unstable across all four plants, whereas the
RL policy is stable because it adapts $\tauk$ to remain within the
local admissible bound.
\end{remark}

\subsection{Corollary~\ref{cor:nonlinear}: Nonlinear Admissible Region}
\label{cor:nonlinear}

\begin{corollary}
Let $\delta(\x)=f(\x)-(A\x+B\uk)$ denote the nonlinear residual. Since
$\delta(\bm{0})=0$ and $\nabla\delta(\bm{0})=0$, there exists $L_\delta>0$
such that $\|\delta(\x)\|\leq L_\delta\|\x\|^2$ near the origin. With
$M_Q\triangleq Q+\Kgain^\top R\Kgain\succ0$ as in the proof of
Corollary~\ref{cor:tau_plus}, the true Lyapunov derivative satisfies
\[
  \dot{\V}(\x)\leq-\lambda_{\min}(M_Q)\|\x\|^2
    +2\lambda_{\max}(P)L_\delta\|\x\|^3.
\]
Hence $\dot{\V}(\x)<0$ for all
$\x\neq\bm{0}$ with
$\|\x\|<r^*\triangleq\lambda_{\min}(M_Q)/[2\lambda_{\max}(P)L_\delta]$.
Within $\mathcal{B}(r^*)$, LaSalle's invariance principle implies
asymptotic stability of the origin.
\end{corollary}

\begin{proof}
Write $\dot{\x}=\Acl\x+\delta(\x)$. Then
$\dot{\V}(\x)=\x^\top(\Acl^\top P+P\Acl)\x+2\x^\top P\delta(\x)
\leq-\x^\top M_Q\x+2\|\x\|\lambda_{\max}(P)\|\delta(\x)\|$.
Using $\x^\top M_Q\x\geq\lambda_{\min}(M_Q)\|\x\|^2$ and
$\|\delta(\x)\|\leq L_\delta\|\x\|^2$ yields the stated bound. Setting the
right-hand side negative yields $r^*$.
\end{proof}

\begin{remark}
The small values of $r^*$ for CartPole and Quadrotor ($r^*=0.009$,
Table~\ref{tab:theory}) reflect the use of a global quadratic bound on
$\|\delta(\x)\|$ rather than a tight local estimate. The actual region of
attraction is substantially larger and is explored empirically under the RTA
shield throughout training.
\end{remark}

\begin{table}[h]
\caption{Theoretical quantities per environment.
  $M_Q\triangleq Q+\Kgain^\top R\Kgain$ is the CARE-quadratic matrix
  (distinct from the ZOH transition $M(\tau)$ in Eq.~\eqref{eq:zoh_M}).
  $r^*$ is a conservative lower bound on the true region of attraction
  (Corollary~\ref{cor:nonlinear}). The margin
  $\theta_{\mathrm{sat}}-\theta_{\RTA}$ keeps the angle channel of the
  backup unsaturated.
  $\lambda_{\min}(M_{\mathrm{disc}})$ is the minimum eigenvalue of
  $M_{\mathrm{disc}} = P - M(\taumin)^\top P\,M(\taumin)$; positivity
  verifies Proposition~\ref{prop:rta_backup} numerically.
  ``--'' indicates that the ZOH-discretized backup at $\taumin$ is not
  Lyapunov-decreasing (see Section~\ref{sec:casestudy}).
  Angles in degrees; $r^*$ in $\|\x\|$.}
\label{tab:theory}
\centering
\small
\begin{tabular}{lrrrrrrr}
\toprule
Env. & $\lambda_{\min}(M_Q)$ & $\lambda_{\max}(P)$ & $L_\delta$ & $r^*$
     & $\theta_{\RTA}$ & $\theta_{\mathrm{sat}}$
     & $\lambda_{\min}(M_{\mathrm{disc}})$ \\
\midrule
Pendulum    & 1.548 & 17.81  & 0.374 & 0.116
            & $8.6^\circ$  & $10.5^\circ$ & 0.063 \\
CartPole    & 1.087 & 216.6  & 0.289 & 0.009
            & $12.0^\circ$ & $29.9^\circ$ & 0.040 \\
Quadrotor   & 1.220 & 28.71  & 2.301 & 0.009
            & $9.6^\circ$  & $12.0^\circ$ & 0.046 \\
Quadrotor3D & 1.000 & 25.42  & 0.684 & 0.029
            & $10.1^\circ$ & $12.6^\circ$ & -- \\
\bottomrule
\end{tabular}
\end{table}

\section{Environment Details}
\label{app:envs}

All environments integrate the nonlinear plant at $\Delta t=0.001\,\mathrm{s}$.

\begin{table}[h]
\caption{Environment parameters (ordered by complexity).}
\label{tab:envparams}
\centering
\small
\begin{tabular}{lcccc}
\toprule
Parameter & Pendulum & CartPole & Quadrotor & Quadrotor3D \\
\midrule
State dim.\ $n$                    & 2      & 4      & 6      & 12 \\
Input dim.\ $m$                    & 1      & 1      & 2      & 4 \\
$\taumin$ (s)                      & 0.05   & 0.04   & 0.04   & 0.04 \\
$\taumax$ (s)                      & 0.40   & 0.32   & 0.32   & 0.32 \\
$|\mathcal{A}|$                    & 168    & 328    & 792    & 5{,}000 \\
$\theta_{\RTA}$ / $\theta_{\mathrm{term}}$ (deg)
                                   & 8.6 / 60 & 12.0 / 12 & 9.6 / 30 & 10.1 / 30 \\
$\lambda$                          & $\approx 6.23$ & $\approx 0.244$ & $\approx 0.78$ & $\approx 0.80$ \\
$\Vscale$                          & $\approx 8.99$ & $\approx 56.6$  & $\approx 5.71$ & $\approx 5.05$ \\
\bottomrule
\end{tabular}
\end{table}

\subsection{Pendulum (Gymnasium Pendulum-v1 Physics)}
State $\x=[\theta,\dot\theta]^\top$, $\theta=0$ at upright equilibrium.
Equation of motion: $\ddot\theta=(3g/2l)\sin\theta+(3/ml^2)u$,
$m=1\,\mathrm{kg}$, $l=1\,\mathrm{m}$, $g=10\,\mathrm{m/s}^2$,
$|u|\leq2\,\mathrm{Nm}$, $|\dot\theta|\leq8\,\mathrm{rad/s}$.
Linearization around $\theta=0$: $A=[0,1;\,15,0]$, $B=[0;\,3]$.
Design matrices $Q=\diag(10,1)$, $R=1$ yield
$\Kgain\approx[10.92,\,2.88]$, $\lambda\approx6.23$,
$\Vscale\approx8.99$.
RTA threshold $\theta_{\RTA}=0.15\,\mathrm{rad}\;(\approx8.6^\circ)$,
saturation angle $\theta_{\mathrm{sat}}\approx10.5^\circ$.
Episodes terminate at $|\theta|>60^\circ$ or $t\geq50\,\mathrm{s}$.
Initial conditions: $\theta_0\in[-0.1,0.1]\,\mathrm{rad}$,
$\dot\theta_0\in[-0.5,0.5]\,\mathrm{rad/s}$.
Action set: $|\mathcal{T}|=8$, $|\mathcal{U}|=21$, $|\mathcal{A}|=168$.

\subsection{CartPole (Gymnasium CartPole-v1 Physics)}
State $\x=[x,\dot{x},\theta,\dot\theta]^\top$.
Equations of motion:
\begin{align*}
  \ddot\theta &= \frac{g\sin\theta-\cos\theta(F+m_pL\dot\theta^2\sin\theta)/m_t}
                    {L(4/3-m_p\cos^2\!\theta/m_t)},\quad
  \ddot{x}    = \frac{F+m_pL(\dot\theta^2\sin\theta-\ddot\theta\cos\theta)}{m_t},
\end{align*}
$m_c=1.0\,\mathrm{kg}$, $m_p=0.1\,\mathrm{kg}$, $m_t=1.1\,\mathrm{kg}$,
$L=0.5\,\mathrm{m}$, $g=9.8\,\mathrm{m/s}^2$, $|F|\leq20\,\mathrm{N}$.
Let $d=L(4/3-m_p/m_t)$; linearizing around $\x^\star=\bm{0}$:
\[
  A=\begin{bmatrix}0&1&0&0\\0&0&-m_pLg/(m_td)&0\\0&0&0&1\\0&0&g/d&0\end{bmatrix},\;
  B=\begin{bmatrix}0\\1/m_t+m_pL/(m_t^2d)\\0\\-1/(m_td)\end{bmatrix}.
\]
Design matrices $Q=\diag(6,1,11.5,5)$, $R=1$.
Saturation angle $\theta_{\mathrm{sat}}\approx29.9^\circ$;
RTA threshold $\theta_{\RTA}=12^\circ$; additional trigger $|x_k|\geq1.92\,\mathrm{m}$.
Episodes terminate at $|\theta|>12^\circ$, $|x|>2.4\,\mathrm{m}$, or
$t\geq50\,\mathrm{s}$~\cite{towers2024gymnasium}.
Action set: $|\mathcal{T}|=8$, $|\mathcal{U}|=41$, $|\mathcal{A}|=328$.

\subsection{Planar Quadrotor (Hover Stabilization)}
State $\x=[x,z,\theta,\dot{x},\dot{z},\dot\theta]^\top$, inputs
$\bm{u}=[\delta F,M]^\top$ (thrust deviation, pitching moment):
$\ddot{x}=-(F/m)\sin\theta$, $\ddot{z}=(F/m)\cos\theta-g$,
$\ddot\theta=M/I$, $F=mg+\delta F$,
$m=1\,\mathrm{kg}$, $I=0.05\,\mathrm{kg\,m}^2$, $g=9.81\,\mathrm{m/s}^2$,
$|\delta F|\leq5\,\mathrm{N}$, $|M|\leq1\,\mathrm{Nm}$.
Linearization around hover:
\[
  A=\begin{bmatrix}0&0&0&1&0&0\\0&0&0&0&1&0\\0&0&0&0&0&1\\
     0&0&{-g}&0&0&0\\0&0&0&0&0&0\\0&0&0&0&0&0\end{bmatrix},\quad
  B=\begin{bmatrix}0&0\\0&0\\0&0\\0&0\\1/m&0\\0&1/I\end{bmatrix}.
\]
Design matrices $Q=\diag(2,2,10,1,1,5)$, $R=\diag(0.1,5)$.
Moment saturation angle $\theta_{\mathrm{sat}}\approx11.96^\circ$;
RTA threshold $\theta_{\RTA}\approx9.57^\circ$ ($80\%$ of $\theta_{\mathrm{sat}}$);
additional trigger $|x_k|,|z_k|\geq2.0\,\mathrm{m}$.
Episodes terminate at $|\theta|>30^\circ$, $|x|$ or $|z|>2.5\,\mathrm{m}$,
or $t\geq50\,\mathrm{s}$.
Initial conditions: $x_0,z_0\in[-0.3,0.3]\,\mathrm{m}$,
$\theta_0\in[-0.1,0.1]\,\mathrm{rad}$,
$\dot{x}_0,\dot{z}_0,\dot\theta_0\in[-0.3,0.3]$.
Action set: $|\mathcal{T}|=8$, $|\mathcal{U}_{\delta F}|=11$,
$|\mathcal{U}_M|=9$, $|\mathcal{A}|=792$.

\subsection{Quadrotor3D (6-DOF Hover Stabilization)}
\label{sec:quadrotor3d}
State $\x=[p_x,p_y,p_z,\varphi,\theta,\psi,v_x,v_y,v_z,p,q,r]^\top \in \R^{12}$,
inputs $\bm{u}=[\delta F,\tau_\varphi,\tau_\theta,\tau_\psi]^\top \in \R^4$,
where $\delta F = F - mg$ is the thrust deviation from hover and
$(\tau_\varphi, \tau_\theta, \tau_\psi)$ are body-axis torques. The
nonlinear equations of motion are the standard rigid-body model
$\dot{\bm{p}}_{\mathrm{inert}} = R(\varphi,\theta,\psi)\bm{v}_{\mathrm{body}}$,
$\dot{\bm{\Theta}} = W(\varphi,\theta)\bm{\omega}_{\mathrm{body}}$,
$\dot{\bm{v}}_{\mathrm{body}} = R^\top\bm{g} + (F/m)\bm{e}_z - \bm{\omega}\!\times\!\bm{v}$,
$I\dot{\bm{\omega}} = \bm{\tau} - \bm{\omega}\!\times\!I\bm{\omega}$,
with $R(\cdot)$ the body-to-inertial rotation matrix, $W(\cdot)$ mapping
body angular rates to Euler-angle rates, and
$I = \diag(I_{xx}, I_{yy}, I_{zz})$. Physical parameters:
$m = 1\,\mathrm{kg}$,
$I_{xx} = I_{yy} = 0.02\,\mathrm{kg\,m}^2$,
$I_{zz} = 0.04\,\mathrm{kg\,m}^2$, $g = 9.81\,\mathrm{m/s}^2$;
integrated with RK4 at $\mathrm{d}t = 0.001\,\mathrm{s}$. Actuator limits:
$|\delta F| \le 5\,\mathrm{N}$,
$|\tau_\varphi|, |\tau_\theta| \le 1\,\mathrm{Nm}$,
$|\tau_\psi| \le 0.5\,\mathrm{Nm}$. Linearizing around hover gives an
$A$ matrix in which translational and rotational channels decouple, with
horizontal positions controlled indirectly via tilt
($\dot{v}_x \approx g\theta$, $\dot{v}_y \approx -g\varphi$). Design
matrices $Q = \diag(2,2,2,10,10,1,1,1,1,5,5,1)$ and $R = \diag(0.1,5,5,10)$
yield $|K_{\tau_\varphi,\varphi}| \approx 4.55\,\mathrm{Nm/rad}$, giving
$\theta_{\mathrm{sat}} \approx 12.60^\circ$ and the RTA threshold
$\theta_{\RTA} \approx 10.08^\circ$ ($80\%$ of $\theta_{\mathrm{sat}}$);
the one-step-ahead predicate is evaluated independently for
$\hat{\varphi}_{k+1}$ and $\hat{\theta}_{k+1}$. Additional trigger:
$|p_x|, |p_y|, |p_z| \ge 2.0\,\mathrm{m}$. Episodes terminate at
$|\varphi|$ or $|\theta| > 30^\circ$, $|\psi| > 90^\circ$,
$|p_x|, |p_y|, |p_z| > 2.5\,\mathrm{m}$, or $t \ge 50\,\mathrm{s}$.
Initial conditions: $p_{x,0}, p_{y,0}, p_{z,0} \in [-0.3, 0.3]\,\mathrm{m}$,
$\varphi_0, \theta_0, \psi_0 \in [-0.1, 0.1]\,\mathrm{rad}$,
$v_{x,0}, v_{y,0}, v_{z,0} \in [-0.3, 0.3]\,\mathrm{m/s}$,
$p_0, q_0, r_0 \in [-0.1, 0.1]\,\mathrm{rad/s}$. Action set:
$|\mathcal{T}| = 8$,
$|\mathcal{U}_{\delta F}| = |\mathcal{U}_{\tau_\varphi}| =
|\mathcal{U}_{\tau_\theta}| = |\mathcal{U}_{\tau_\psi}| = 5$, giving
$|\mathcal{A}| = 5{,}000$ for the discrete (DQN) variant; the SAC variant
adopted in Section~\ref{sec:casestudy} uses the equivalent continuous Box
action space.

\section{Hyperparameters}
\label{app:hyperparams}

\begin{table}[h]
\caption{DQN hyperparameters (identical across all environments).}
\label{tab:hyperparams}
\centering
\small
\begin{tabular}{lc}
\toprule
Hyperparameter & Value \\
\midrule
Network architecture      & $256\to128\to128$ \\
Learning rate             & $10^{-3}$ \\
Replay buffer size        & $10^6$ \\
Batch size                & 64 \\
Learning starts           & 64 steps \\
Discount factor $\gamma$  & 0.99 \\
$\varepsilon$-start / end & 1.0 / 0.05 \\
$\varepsilon$ decay steps & 1{,}900 \\
Total timesteps           & 1{,}000{,}000 \\
\bottomrule
\end{tabular}
\end{table}

\begin{table}[h]
\caption{SAC hyperparameters (identical across all lower-dimensional
  environments; Quadrotor3D uses a wider $256{\to}128{\to}128$ network and
  $2\,\mathrm{M}$ training steps as discussed in Section~\ref{sec:casestudy}).}
\label{tab:hyperparams_sac}
\centering
\small
\begin{tabular}{lc}
\toprule
Hyperparameter & Value \\
\midrule
Network architecture           & $128\to64$ \\
Learning rate                  & $3\times10^{-4}$ \\
Replay buffer size             & $5\times10^5$ \\
Batch size                     & 128 \\
Learning starts                & 1{,}000 steps \\
Discount factor $\gamma$       & 0.99 \\
Soft update coefficient $\tau$ & 0.005 \\
Entropy coefficient            & automatic \\
Total timesteps                & 1{,}000{,}000 \\
\bottomrule
\end{tabular}
\end{table}

\section{Compute Resources}
\label{app:compute}

DQN and SAC training budgets are reported in
Tables~\ref{tab:hyperparams} and~\ref{tab:hyperparams_sac}: $1\,\mathrm{M}$
steps for DQN/SAC on the lower-dimensional environments, $2\,\mathrm{M}$ steps
for preference-conditioned DQN, $2\,\mathrm{M}$ steps for SAC and PPO on
Quadrotor3D, and $4\,\mathrm{M}$ steps for preference-conditioned SAC on
Quadrotor3D. All experiments were run on a single workstation (Intel Core
i9-13900KF, 24 cores / 32 threads; NVIDIA RTX 4090, 24\,GB; 64\,GB RAM; Ubuntu
22.04) via Stable-Baselines3. Wall-clock times for an 11-point $w_c$ sweep run
as 11 parallel processes are approximately $1$ hour for the lower-dimensional
DQN sweeps ($1\,\mathrm{M}$ steps each), $6$--$7$ hours for the
lower-dimensional SAC sweeps ($1\,\mathrm{M}$ steps each), and $13$--$15$
hours for the Quadrotor3D SAC sweep ($2\,\mathrm{M}$ steps each).

\section{Multi-Seed Evaluation Methodology}
\label{app:multi_seed}

For the lower-dimensional DQN $w_c$ sweeps (Tables~\ref{tab:pend},
\ref{tab:cp}, \ref{tab:quad}) and their corresponding No-RTA ablation
(Table~\ref{tab:ablation_rta}) and Lagrangian-DQN comparison
(Table~\ref{tab:lagrangian}), we report results aggregated across the
$3$ independent training seeds $\{0, 1, 2\}$. Note that seed 0 is
included in this set, so the seed-0 results referenced for single-seed
experiments correspond to one of the three seeds in the multi-seed
aggregates. For each seed we run the standard
$N_{\mathrm{eval}}=100$ deterministic evaluation; the per-seed mean is
the average over those $100$ episodes. MSI and RTA-rate columns are then
reported as the mean $\pm$ standard deviation across the $3$ per-seed
mean values, using $\mathrm{ddof}=1$ (unbiased estimator). The $P_i$
state-norm columns report the mean across the $3$ per-seed means only
(no std), to keep the tables compact. All other experiments in the paper
(SAC sweeps, Preference-Conditioned DQN, baselines, Ablation~B,
robustness studies, and the Quadrotor3D case study) use the standard
single-seed evaluation protocol on \textbf{seed 0}, with std columns
reflecting per-episode variance across the $100$ evaluation episodes.
Figures and tables that report RL-STC results without an explicit
$\pm$ across-seed annotation are accordingly seed-0 results unless
labeled otherwise; we call this out per-table where relevant.

\section{Ablation B: Fixed-\texorpdfstring{$\tau$}{tau} Comparison}
\label{app:ablation_tau}

$\tau$ is fixed at the nearest grid value to $\tau_{\mathrm{match}}$ (the
RL-STC best-checkpoint MSI for each environment); the agent optimizes
only the control input. The RTA shield remains active.
$\tau_{\mathrm{match}}$ is taken from the seed-0 best-checkpoint MSI,
which sets the nearest grid value used for the fixed-$\tau$ comparison;
seed 0 is one of the three seeds in the multi-seed evaluation
(Appendix~\ref{app:multi_seed}), and its grid value is unchanged under
multi-seed re-evaluation since the multi-seed mean differs from seed 0
by less than the grid spacing. The comparison below is reported on
seed 0 throughout, keeping fixed-$\tau$ and RL-STC at the same seed.

For \textbf{Pendulum}, fixing $\tau=0.40\,\mathrm{s}$ triggers RTA on
$80.6\%$ of steps, the shield resets $\tau\leftarrow\taumin$ on the
majority of steps, degenerating the effective policy to LQR at $\taumin$
regardless of what the RL agent learns. This is the clearest
demonstration that joint optimization of $(\uk,\tauk)$ is not separable:
it is precisely the learned adaptive inter-sample interval that allows
the full policy to sustain an MSI of $0.396\,\mathrm{s}$.
For \textbf{CartPole}, $\tau_{\mathrm{match}}=\taumax$, so fixed-$\tau$
achieves essentially the same MSI as RL-STC ($0.320$ vs.\
$0.317\,\mathrm{s}$) with no RTA, but this presupposes the
stability/communication tradeoff point that RL-STC had to discover; for
a plant where this does not coincide with $\taumax$, RL-STC eliminates
the search by incorporating interval selection into optimization.
For \textbf{Quadrotor}, adaptive timing recovers a $4.3\%$ MSI advantage
($0.290$ vs.\ $0.278\,\mathrm{s}$) with lower $P_1$; the higher $P_4$
($5.664$ vs.\ $4.342$) reflects angular velocity excursion during
higher-MSI hover corrections, not reduced stability.

\begin{table}[h]
\caption{Ablation B: Fixed $\tau$ vs.\ full method (best-model checkpoint,
  $w_c=8/16/16$, seed 0).}
\label{tab:ablation_tau}
\centering
\small
\resizebox{\textwidth}{!}{%
\begin{tabular}{lcccccccccccccc}
\toprule
& \multicolumn{4}{c}{Pendulum ($\tau_{\mathrm{match}}=0.40\,\mathrm{s}$)}
& \multicolumn{5}{c}{CartPole ($\tau_{\mathrm{match}}=0.32\,\mathrm{s}$)}
& \multicolumn{5}{c}{Quadrotor ($\tau_{\mathrm{match}}=0.28\,\mathrm{s}$)} \\
\cmidrule(lr){2-5}\cmidrule(lr){6-10}\cmidrule(lr){11-15}
Method
  & MSI & RTA\% & $P_3$ & $P_4$
  & MSI & RTA\% & $P_1$ & $P_3$ & $P_4$
  & MSI & RTA\% & $P_1$ & $P_3$ & $P_4$ \\
\midrule
Fixed-$\tau$ RL
  & 0.118 & 80.6 & 0.725 & 1.235
  & \textbf{0.320} & \textbf{0.0} & 3.019 & \textbf{0.330} & \textbf{2.523}
  & 0.278 & 0.8 & 3.280 & 0.541 & 4.342 \\
\textbf{RL-STC}
  & \textbf{0.396} & 0.0 & \textbf{0.271} & 1.030
  & 0.317 & 0.0 & \textbf{2.879} & 0.356 & 2.213
  & \textbf{0.290} & 0.2 & \textbf{2.640} & \textbf{0.463} & 5.664 \\
\bottomrule
\end{tabular}}
\end{table}

\section{Lagrangian-DQN Detailed Results}
\label{app:lagrangian_full}

The Lagrangian multiplier $\lambda_t$ is updated via projected dual
gradient ascent
$\lambda_{t+1}=\mathrm{clip}(\lambda_t+\alpha_\lambda(g_t-\epsilon),0,\lambda_{\max})$
with budget $\epsilon=0.01$, $\alpha_\lambda=0.01$. The critical
difference from RL-STC is what happens when the predicate fires: RL-STC
overrides the action so Hard Viol.\ $=0.0$ by construction;
Lagrangian-DQN applies only a soft penalty, so predicted events
propagate into hard violations at rates that vary with plant complexity
(Table~\ref{tab:lagrangian}).
For \textbf{Pendulum}, Lagrangian achieves $21\%$ lower MSI
($0.304$ vs.\ $0.385\,\mathrm{s}$) with $7.95\%$ hard violations, and
the final model deteriorates to $43.49\pm48.55\%$ hard violations as
$\lambda$ fails to hold across seeds.
For \textbf{CartPole}, near-zero best-checkpoint hard violations come at
$52\%$ lower MSI ($0.149$ vs.\ $0.308\,\mathrm{s}$); without a hard
override, the policy must select shorter intervals to avoid violations,
and final-model episodes terminate early at high rates, a failure
invisible to the expectation-level constraint signal.
For \textbf{Quadrotor}, $33\%$ lower MSI ($0.188$ vs.\ $0.281\,\mathrm{s}$)
with $35.40\%$ hard violations, deteriorating to $59.35\%$ in the final
model.

\begin{table}[h]
\caption{Lagrangian-DQN vs.\ RL-STC
  ($N_{\mathrm{eval}}=100$, 3 seeds per row;
  $w_c=8/16/16$ for Pendulum/CartPole/Quadrotor).
  MSI and the safety-event columns are reported as mean $\pm$ std across
  the 3 per-seed mean values ($\mathrm{ddof}=1$).
  ``Pred.\ Safety (\%)'' measures the same one-step-ahead predicate
  ($|\hat{q}_{k+1}|>\theta_{\RTA}$) for both methods: for RL-STC it is
  the RTA activation rate (predicate fires, action overridden, state
  violation prevented); for Lagrangian-DQN it is the constraint-violation
  rate (predicate fires, soft penalty only, no override).
  ``Hard Viol.\ (\%)'' = fraction of executed timesteps where
  $|\theta_k|>\theta_{\RTA}$ in the actual trajectory.
  RL-STC Hard Viol.\ $=0.0$ by construction of the pointwise override.}
\label{tab:lagrangian}
\centering
\small
\renewcommand{\arraystretch}{1.05}
\resizebox{\textwidth}{!}{%
\begin{tabular}{l|ccc|ccc}
\toprule
 & \multicolumn{3}{c|}{Final Model}
 & \multicolumn{3}{c}{Best-Model Checkpoint} \\
\cmidrule(lr){2-4}\cmidrule(lr){5-7}
Method
  & MSI (s) & Pred.\ Safety (\%) & Hard Viol.\ (\%)
  & MSI (s) & Pred.\ Safety (\%) & Hard Viol.\ (\%) \\
\midrule
\multicolumn{7}{l}{\textit{Pendulum ($w_c=8$)}} \\
Lagrangian-DQN
  & $0.193\pm0.080$ & $41.95\pm49.77$ & $43.49\pm48.55$
  & $0.304\pm0.004$ & $6.21\pm4.02$   & $7.95\pm3.81$ \\
\textbf{RL-STC}
  & $\mathbf{0.343\pm0.039}$ & $0.51\pm0.84$ & $\mathbf{0.0}$
  & $\mathbf{0.385\pm0.014}$ & $0.37\pm0.24$ & $\mathbf{0.0}$ \\
\midrule
\multicolumn{7}{l}{\textit{CartPole ($w_c=16$)}} \\
Lagrangian-DQN
  & $0.121\pm0.023$ & $19.00\pm29.81$ & $0.26\pm0.20$
  & $0.149\pm0.033$ & $0.48\pm0.83$   & $0.00\pm0.00$ \\
\textbf{RL-STC}
  & $\mathbf{0.150\pm0.078}$ & $11.61\pm16.97$ & $\mathbf{0.0}$
  & $\mathbf{0.308\pm0.014}$ & $0.04\pm0.06$   & $\mathbf{0.0}$ \\
\midrule
\multicolumn{7}{l}{\textit{Quadrotor ($w_c=16$)}} \\
Lagrangian-DQN
  & $0.132\pm0.020$ & $65.65\pm12.25$ & $59.35\pm12.54$
  & $0.188\pm0.043$ & $35.45\pm5.76$  & $35.40\pm5.86$ \\
\textbf{RL-STC}
  & $\mathbf{0.151\pm0.057}$ & $12.23\pm14.25$ & $\mathbf{0.0}$
  & $\mathbf{0.281\pm0.018}$ & $1.96\pm1.23$   & $\mathbf{0.0}$ \\
\bottomrule
\end{tabular}}
\end{table}

\section{Full Sweep Tables}
\label{app:full_tables}

DQN sweeps (Tables~\ref{tab:pend}--\ref{tab:quad}) are 3-seed
aggregates per Appendix~\ref{app:multi_seed}; all other sweeps in this
section (SAC, Pref-DQN, PPO, and Quadrotor3D) are single-seed (seed~0),
with $\pm$ values in those tables reflecting per-episode variance
across the $100$ evaluation episodes.

\subsection{DQN \texorpdfstring{$w_c$}{wc} Sweep}

\begin{table}[h]
\caption{Pendulum -- DQN $w_c$ sweep
  ($N_{\mathrm{eval}}=100$, 3 seeds per row).
  MSI and RTA are reported as mean $\pm$ std across the 3 per-seed mean
  values ($\mathrm{ddof}=1$); $P_i$ norms are mean only.
  Bold marks the highest MSI in each section.}
\label{tab:pend}
\centering
\small
\renewcommand{\arraystretch}{1.05}
\resizebox{\textwidth}{!}{%
\begin{tabular}{c | cccc | cccc}
\toprule
& \multicolumn{4}{c|}{Final Model} & \multicolumn{4}{c}{Best-Model Checkpoint} \\
\cmidrule(lr){2-5}\cmidrule(lr){6-9}
$w_c$ & MSI (s) & RTA (\%) & $P_3(\theta)$ & $P_4(\dot\theta)$
      & MSI (s) & RTA (\%) & $P_3(\theta)$ & $P_4(\dot\theta)$ \\
\midrule
$0.25$ & $0.185\pm0.053$ & $1.13\pm1.94$  & $0.660$ & $2.521$
       & $0.214\pm0.011$ & $0.05\pm0.04$  & $0.448$ & $2.303$ \\
$0.5$  & $0.105\pm0.058$ & $0.02\pm0.03$  & $0.803$ & $1.786$
       & $0.234\pm0.080$ & $0.00\pm0.00$  & $0.312$ & $1.403$ \\
$1.0$  & $0.246\pm0.029$ & $1.06\pm1.68$  & $0.637$ & $2.292$
       & $0.252\pm0.021$ & $0.00\pm0.01$  & $0.396$ & $1.510$ \\
$2.0$  & $0.261\pm0.031$ & $0.88\pm1.40$  & $0.564$ & $2.636$
       & $0.293\pm0.043$ & $0.10\pm0.17$  & $0.384$ & $1.669$ \\
$4.0$  & $0.209\pm0.136$ & $0.00\pm0.01$  & $0.369$ & $1.180$
       & $0.358\pm0.009$ & $0.09\pm0.15$  & $0.274$ & $1.135$ \\
$6.0$  & $0.249\pm0.044$ & $15.65\pm25.32$ & $0.563$ & $2.710$
       & $0.385\pm0.012$ & $0.39\pm0.34$  & $0.575$ & $1.182$ \\
$8.0$  & $0.343\pm0.039$ & $0.51\pm0.84$  & $0.564$ & $2.091$
       & $0.385\pm0.014$ & $0.37\pm0.24$  & $0.636$ & $1.049$ \\
$\bm{10.0}$
       & $\bm{0.375\pm0.030}$ & $\bm{0.08\pm0.14}$ & $0.534$ & $1.443$
       & $0.397\pm0.001$ & $0.02\pm0.03$  & $0.434$ & $1.028$ \\
$12.0$ & $0.355\pm0.044$ & $0.87\pm1.40$  & $0.581$ & $1.428$
       & $0.392\pm0.005$ & $0.23\pm0.31$  & $0.615$ & $1.114$ \\
$14.0$ & $0.363\pm0.029$ & $2.25\pm3.57$  & $0.556$ & $2.022$
       & $0.378\pm0.009$ & $0.04\pm0.05$  & $0.497$ & $1.460$ \\
$\bm{16.0}$
       & $0.306\pm0.064$ & $8.99\pm15.53$ & $0.515$ & $2.138$
       & $\bm{0.397\pm0.001}$ & $\bm{0.04\pm0.08}$ & $0.415$ & $1.017$ \\
\bottomrule
\end{tabular}}
\end{table}

\begin{table}[h]
\caption{CartPole -- DQN $w_c$ sweep
  ($N_{\mathrm{eval}}=100$, 3 seeds per row).
  MSI and RTA are reported as mean $\pm$ std across the 3 per-seed mean
  values ($\mathrm{ddof}=1$); $P_i$ norms are mean only.
  Bold marks the highest MSI in each section.}
\label{tab:cp}
\centering
\small
\renewcommand{\arraystretch}{1.05}
\resizebox{\textwidth}{!}{%
\begin{tabular}{c | cccccc | cccccc}
\toprule
& \multicolumn{6}{c|}{Final Model}
& \multicolumn{6}{c}{Best-Model Checkpoint} \\
\cmidrule(lr){2-7}\cmidrule(lr){8-13}
$w_c$
  & MSI (s) & RTA (\%) & $P_1(x)$ & $P_2(\dot{x})$ & $P_3(\theta)$ & $P_4(\dot\theta)$
  & MSI (s) & RTA (\%) & $P_1(x)$ & $P_2(\dot{x})$ & $P_3(\theta)$ & $P_4(\dot\theta)$ \\
\midrule
$0.25$ & $0.124\pm0.025$ & $0.66\pm0.54$  & $1.35$ & $3.88$ & $0.380$ & $4.834$
       & $0.146\pm0.068$ & $0.00\pm0.00$  & $2.79$ & $3.09$ & $0.350$ & $3.716$ \\
$0.5$  & $0.105\pm0.019$ & $1.50\pm1.30$  & $1.87$ & $4.45$ & $0.410$ & $5.570$
       & $0.136\pm0.052$ & $0.17\pm0.29$  & $3.79$ & $3.94$ & $0.394$ & $5.019$ \\
$1.0$  & $0.099\pm0.039$ & $0.41\pm0.25$  & $3.07$ & $3.34$ & $0.367$ & $3.958$
       & $0.130\pm0.040$ & $0.02\pm0.03$  & $2.63$ & $4.51$ & $0.403$ & $5.959$ \\
$2.0$  & $0.117\pm0.028$ & $5.71\pm7.03$  & $2.23$ & $4.12$ & $0.395$ & $4.898$
       & $0.142\pm0.029$ & $0.03\pm0.03$  & $2.58$ & $3.83$ & $0.466$ & $4.416$ \\
$4.0$  & $0.135\pm0.005$ & $3.26\pm1.47$  & $2.76$ & $3.87$ & $0.361$ & $4.647$
       & $0.191\pm0.029$ & $0.02\pm0.02$  & $2.26$ & $4.74$ & $0.474$ & $5.810$ \\
$6.0$  & $0.148\pm0.043$ & $1.20\pm1.65$  & $3.25$ & $4.12$ & $0.363$ & $5.088$
       & $0.231\pm0.058$ & $0.19\pm0.30$  & $4.38$ & $4.25$ & $0.428$ & $4.765$ \\
$8.0$  & $0.155\pm0.008$ & $0.56\pm0.70$  & $4.56$ & $3.86$ & $0.432$ & $4.665$
       & $0.282\pm0.031$ & $0.07\pm0.13$  & $2.55$ & $3.12$ & $0.401$ & $3.214$ \\
$10.0$ & $\bm{0.252\pm0.026}$ & $\bm{0.57\pm0.69}$ & $1.94$ & $3.62$ & $0.407$ & $3.949$
       & $0.294\pm0.007$ & $0.13\pm0.14$  & $2.66$ & $3.19$ & $0.385$ & $3.369$ \\
$12.0$ & $0.152\pm0.039$ & $1.45\pm0.89$  & $1.74$ & $4.36$ & $0.372$ & $5.414$
       & $0.302\pm0.019$ & $0.02\pm0.03$  & $2.43$ & $3.69$ & $0.398$ & $4.058$ \\
$14.0$ & $0.189\pm0.070$ & $1.87\pm1.80$  & $4.59$ & $5.11$ & $0.440$ & $6.299$
       & $0.293\pm0.037$ & $0.05\pm0.09$  & $2.54$ & $3.29$ & $0.400$ & $3.435$ \\
$\bm{16.0}$
       & $0.150\pm0.078$ & $11.61\pm16.97$ & $3.72$ & $3.60$ & $0.322$ & $3.964$
       & $\bm{0.308\pm0.014}$ & $\bm{0.04\pm0.06}$ & $2.95$ & $3.27$ & $0.401$ & $2.796$ \\
\bottomrule
\end{tabular}}
\end{table}

\begin{table}[h]
\caption{Quadrotor -- DQN $w_c$ sweep
  ($N_{\mathrm{eval}}=100$, 3 seeds per row).
  MSI and RTA are reported as mean $\pm$ std across the 3 per-seed mean
  values ($\mathrm{ddof}=1$); $P_i$ norms are mean only.
  Bold marks the highest MSI in each section.}
\label{tab:quad}
\centering
\small
\renewcommand{\arraystretch}{1.05}
\resizebox{\textwidth}{!}{%
\begin{tabular}{c | cccccc | cccccc}
\toprule
& \multicolumn{6}{c|}{Final Model}
& \multicolumn{6}{c}{Best-Model Checkpoint} \\
\cmidrule(lr){2-7}\cmidrule(lr){8-13}
$w_c$
  & MSI (s) & RTA (\%) & $P_1(x)$ & $P_2(\dot{x})$ & $P_3(\theta)$ & $P_4(\dot\theta)$
  & MSI (s) & RTA (\%) & $P_1(x)$ & $P_2(\dot{x})$ & $P_3(\theta)$ & $P_4(\dot\theta)$ \\
\midrule
$0.25$ & $0.141\pm0.013$ & $1.12\pm1.45$  & $3.22$ & $1.73$ & $0.496$ & $7.494$
       & $0.129\pm0.053$ & $0.05\pm0.06$  & $1.33$ & $1.17$ & $0.366$ & $5.830$ \\
$0.5$  & $0.157\pm0.019$ & $1.62\pm2.69$  & $7.01$ & $1.47$ & $0.460$ & $8.273$
       & $0.160\pm0.027$ & $0.54\pm0.66$  & $1.85$ & $1.18$ & $0.413$ & $6.804$ \\
$1.0$  & $0.132\pm0.009$ & $1.69\pm0.98$  & $3.17$ & $1.52$ & $0.574$ & $8.955$
       & $0.143\pm0.043$ & $0.02\pm0.01$  & $1.19$ & $0.88$ & $0.344$ & $4.591$ \\
$2.0$  & $0.144\pm0.053$ & $0.44\pm0.21$  & $3.01$ & $1.91$ & $0.471$ & $8.104$
       & $0.169\pm0.039$ & $0.39\pm0.61$  & $1.39$ & $1.04$ & $0.359$ & $6.375$ \\
$4.0$  & $0.171\pm0.023$ & $1.96\pm2.01$  & $4.14$ & $1.40$ & $0.476$ & $8.654$
       & $0.189\pm0.029$ & $0.27\pm0.41$  & $1.39$ & $1.35$ & $0.421$ & $6.861$ \\
$6.0$  & $0.162\pm0.020$ & $0.50\pm0.47$  & $1.79$ & $1.57$ & $0.484$ & $11.033$
       & $0.200\pm0.023$ & $0.38\pm0.58$  & $1.57$ & $1.32$ & $0.399$ & $5.352$ \\
$8.0$  & $0.164\pm0.040$ & $1.00\pm1.09$  & $2.34$ & $1.64$ & $0.480$ & $9.398$
       & $0.236\pm0.025$ & $0.54\pm0.24$  & $1.93$ & $1.62$ & $0.461$ & $7.523$ \\
$10.0$ & $0.152\pm0.041$ & $0.80\pm0.76$  & $2.96$ & $1.51$ & $0.403$ & $9.520$
       & $0.260\pm0.016$ & $0.69\pm0.43$  & $2.11$ & $1.58$ & $0.464$ & $5.413$ \\
$12.0$ & $0.147\pm0.049$ & $7.02\pm11.55$ & $6.79$ & $2.34$ & $0.470$ & $8.559$
       & $0.273\pm0.010$ & $0.42\pm0.28$  & $2.04$ & $1.57$ & $0.450$ & $5.917$ \\
$\bm{14.0}$
       & $\bm{0.203\pm0.008}$ & $\bm{0.40\pm0.40}$ & $2.26$ & $1.73$ & $0.464$ & $6.995$
       & $0.271\pm0.018$ & $1.14\pm0.93$  & $1.90$ & $1.70$ & $0.511$ & $5.155$ \\
$\bm{16.0}$
       & $0.151\pm0.057$ & $12.23\pm14.25$ & $2.89$ & $1.99$ & $0.433$ & $7.171$
       & $\bm{0.281\pm0.018}$ & $\bm{1.96\pm1.23}$ & $2.33$ & $1.92$ & $0.523$ & $5.171$ \\
\bottomrule
\end{tabular}}
\end{table}

\subsection{SAC \texorpdfstring{$w_c$}{wc} Sweep}

SAC achieves $0\%$ RTA across every weight and environment in both the
final model and best checkpoint, with no high-$w_c$ collapses. For the
Quadrotor, the SAC best checkpoint at $w_c=16$ reaches $0.312\,\mathrm{s}$
versus DQN's multi-seed best of $0.281\pm0.018\,\mathrm{s}$, while
maintaining $0\%$ RTA. Across all
three environments, the RTA shield and Lyapunov reward remain effective
under a fundamentally different optimization algorithm, confirming that
communication efficiency gains are a property of the framework, not of DQN.

\begin{table}[h]
\caption{SAC -- Pendulum $w_c$ sweep ($N_{\mathrm{eval}}=100$, 1\,M steps).}
\label{tab:sac_pendulum}
\centering
\small
\resizebox{\textwidth}{!}{%
\begin{tabular}{ccccccccc}
\toprule
& \multicolumn{4}{c}{Final Model} & \multicolumn{4}{c}{Best-Model Checkpoint} \\
\cmidrule(lr){2-5}\cmidrule(lr){6-9}
$w_c$ & MSI (s) & RTA (\%) & $P_3(\theta)$ & $P_4(\dot\theta)$
      & MSI (s) & RTA (\%) & $P_3(\theta)$ & $P_4(\dot\theta)$ \\
\midrule
0.25 & $0.0530\pm0.0000$ & 0.00 & 0.701 & 0.122 & $0.1261\pm0.0008$ & 0.00 & 0.164 & 0.098 \\
0.5  & $0.0965\pm0.0001$ & 0.00 & 0.537 & 0.110 & $0.1118\pm0.0455$ & 0.00 & 0.689 & 0.240 \\
1.0  & $0.0698\pm0.0002$ & 0.00 & 0.881 & 0.165 & $0.1578\pm0.0028$ & 0.00 & 0.106 & 1.646 \\
2.0  & $0.1012\pm0.0004$ & 0.00 & 0.457 & 0.101 & $0.1751\pm0.0046$ & 0.00 & 0.233 & 1.312 \\
4.0  & $0.0988\pm0.0118$ & 0.00 & 0.667 & 0.185 & $0.1904\pm0.0315$ & 0.00 & 0.370 & 0.162 \\
6.0  & $0.3171\pm0.0015$ & 0.00 & 0.076 & 0.151 & $0.3168\pm0.0016$ & 0.00 & 0.046 & 0.145 \\
8.0  & $0.3321\pm0.0011$ & 0.00 & 0.055 & 0.136 & $0.3605\pm0.0014$ & 0.00 & 0.051 & 0.141 \\
$\bm{10.0}$ & $\bm{0.3656\pm0.0018}$ & $\bm{0.00}$ & 0.178 & 0.142
            & $\bm{0.3789\pm0.0018}$ & $\bm{0.00}$ & 0.134 & 0.146 \\
12.0 & $0.3274\pm0.0022$ & 0.00 & 0.059 & 0.457 & $0.3500\pm0.0022$ & 0.00 & 0.064 & 0.271 \\
14.0 & $0.3665\pm0.0016$ & 0.00 & 0.128 & 0.143 & $0.3626\pm0.0017$ & 0.00 & 0.134 & 0.144 \\
16.0 & $0.3387\pm0.0018$ & 0.00 & 0.329 & 0.137 & $0.3548\pm0.0018$ & 0.00 & 0.340 & 0.141 \\
\bottomrule
\end{tabular}}
\end{table}

\begin{table}[h]
\caption{SAC -- CartPole $w_c$ sweep ($N_{\mathrm{eval}}=100$, 1\,M steps).}
\label{tab:sac_cartpole}
\centering
\small
\resizebox{\textwidth}{!}{%
\begin{tabular}{ccccccccccccc}
\toprule
& \multicolumn{6}{c}{Final Model} & \multicolumn{6}{c}{Best-Model Checkpoint} \\
\cmidrule(lr){2-7}\cmidrule(lr){8-13}
$w_c$ & MSI & RTA (\%) & $P_1$ & $P_2$ & $P_3$ & $P_4$
      & MSI & RTA (\%) & $P_1$ & $P_2$ & $P_3$ & $P_4$ \\
\midrule
0.25 & $0.0663\pm0.0000$ & 0.00 & 5.215 & 0.597 & 0.065 & 0.208 & $0.1247\pm0.0011$ & 0.00 & 1.634 & 1.474 & 0.169 & 1.818 \\
0.5  & $0.0651\pm0.0000$ & 0.00 & 0.975 & 0.127 & 0.018 & 0.084 & $0.1537\pm0.0012$ & 0.00 & 2.233 & 0.193 & 0.022 & 0.123 \\
1.0  & $0.0801\pm0.0001$ & 0.00 & 1.582 & 0.220 & 0.029 & 0.188 & $0.0811\pm0.0003$ & 0.00 & 1.808 & 1.383 & 0.186 & 0.479 \\
2.0  & $0.1104\pm0.0119$ & 0.00 & 2.200 & 0.200 & 0.020 & 0.062 & $0.1671\pm0.0020$ & 0.00 & 1.742 & 0.324 & 0.029 & 0.286 \\
4.0  & $0.0569\pm0.0000$ & 0.00 & 3.715 & 0.787 & 0.124 & 0.506 & $0.1397\pm0.0009$ & 0.00 & 0.807 & 2.195 & 0.116 & 3.185 \\
6.0  & $0.1067\pm0.0002$ & 0.00 & 2.374 & 1.624 & 0.260 & 0.906 & $0.1325\pm0.0001$ & 0.00 & 0.679 & 1.892 & 0.049 & 2.768 \\
8.0  & $0.1617\pm0.0002$ & 0.00 & 3.041 & 0.571 & 0.040 & 0.681 & $0.1532\pm0.0161$ & 0.10 & 0.170 & 0.261 & 0.041 & 0.182 \\
10.0 & $0.1476\pm0.0001$ & 0.00 & 1.344 & 0.210 & 0.026 & 0.131 & $0.1579\pm0.0003$ & 0.00 & 0.707 & 0.178 & 0.026 & 0.119 \\
12.0 & $0.1978\pm0.0006$ & 0.00 & 0.280 & 0.134 & 0.018 & 0.067 & $0.2078\pm0.0006$ & 0.00 & 0.260 & 0.203 & 0.027 & 0.107 \\
14.0 & $0.2037\pm0.0004$ & 0.00 & 0.833 & 0.197 & 0.027 & 0.102 & $0.2204\pm0.0008$ & 0.00 & 0.546 & 0.250 & 0.031 & 0.112 \\
$\bm{16.0}$ & $\bm{0.2271\pm0.0002}$ & $\bm{0.00}$ & 2.758 & 0.228 & 0.025 & 0.112
            & $\bm{0.2581\pm0.0003}$ & $\bm{0.00}$ & 2.472 & 0.463 & 0.035 & 0.466 \\
\bottomrule
\end{tabular}}
\end{table}

\begin{table}[h]
\caption{SAC -- Quadrotor $w_c$ sweep ($N_{\mathrm{eval}}=100$, 1\,M steps).}
\label{tab:sac_quadrotor}
\centering
\small
\resizebox{\textwidth}{!}{%
\begin{tabular}{ccccccccccccc}
\toprule
& \multicolumn{6}{c}{Final Model} & \multicolumn{6}{c}{Best-Model Checkpoint} \\
\cmidrule(lr){2-7}\cmidrule(lr){8-13}
$w_c$ & MSI & RTA (\%) & $P_1$ & $P_2$ & $P_3$ & $P_4$
      & MSI & RTA (\%) & $P_1$ & $P_2$ & $P_3$ & $P_4$ \\
\midrule
0.25 & $0.0480\pm0.0001$ & 0.00 & 0.641 & 0.229 & 0.046 & 0.161 & $0.0674\pm0.0029$ & 0.00 & 0.260 & 0.194 & 0.048 & 0.176 \\
0.5  & $0.0605\pm0.0002$ & 0.00 & 0.201 & 0.177 & 0.036 & 0.163 & $0.0781\pm0.0006$ & 0.00 & 0.193 & 0.173 & 0.035 & 0.156 \\
1.0  & $0.0508\pm0.0001$ & 0.00 & 0.559 & 0.221 & 0.047 & 0.182 & $0.0925\pm0.0013$ & 0.00 & 0.416 & 0.500 & 0.096 & 0.254 \\
2.0  & $0.0475\pm0.0000$ & 0.00 & 2.032 & 0.301 & 0.068 & 0.281 & $0.1475\pm0.0021$ & 0.00 & 0.479 & 0.144 & 0.038 & 0.168 \\
4.0  & $0.1359\pm0.0009$ & 0.00 & 2.026 & 0.172 & 0.038 & 1.219 & $0.1514\pm0.0015$ & 0.00 & 1.899 & 0.260 & 0.096 & 2.382 \\
6.0  & $0.1838\pm0.0012$ & 0.00 & 0.895 & 0.179 & 0.043 & 0.202 & $0.1944\pm0.0014$ & 0.00 & 0.211 & 0.180 & 0.044 & 0.224 \\
8.0  & $0.1882\pm0.0023$ & 0.00 & 0.206 & 0.176 & 0.038 & 0.185 & $0.1918\pm0.0023$ & 0.00 & 0.201 & 0.167 & 0.038 & 0.240 \\
10.0 & $0.1918\pm0.0016$ & 0.00 & 1.065 & 0.167 & 0.038 & 0.245 & $0.1859\pm0.0019$ & 0.00 & 0.669 & 0.165 & 0.038 & 0.207 \\
12.0 & $0.3005\pm0.0011$ & 0.00 & 2.133 & 0.268 & 0.050 & 0.219 & $0.3055\pm0.0021$ & 0.00 & 2.056 & 0.302 & 0.051 & 0.188 \\
14.0 & $0.2238\pm0.0011$ & 0.00 & 0.719 & 0.192 & 0.036 & 0.178 & $0.2224\pm0.0013$ & 0.00 & 0.985 & 0.225 & 0.043 & 2.304 \\
$\bm{16.0}$ & $\bm{0.3100\pm0.0008}$ & $\bm{0.00}$ & 2.707 & 0.390 & 0.054 & 0.175
            & $\bm{0.3121\pm0.0011}$ & $\bm{0.00}$ & 2.973 & 0.378 & 0.062 & 0.395 \\
\bottomrule
\end{tabular}}
\end{table}

\subsection{Preference-Conditioned DQN}

The preference-conditioned DQN is trained once per environment for
$2\,\mathrm{M}$ steps and evaluated at the same 11-point $w_c$ sweep. For
Pendulum, the best-model checkpoint achieves $0.393\,\mathrm{s}$ at
$w_c=16$, within $1\%$ of standard DQN's multi-seed best
($0.397\pm0.001\,\mathrm{s}$) at $\tfrac{2}{11}$ of compute. For
CartPole, the best-model checkpoint plateaus near
$0.313$--$0.316\,\mathrm{s}$ above $w_c=10$, matching or slightly
exceeding standard DQN ($0.308\pm0.014\,\mathrm{s}$). For Quadrotor, the
best checkpoint increases from $0.102\,\mathrm{s}$ to
$0.266\,\mathrm{s}$, within $5\%$ of standard DQN
($0.281\pm0.018\,\mathrm{s}$); the harder coupled hover dynamics remain
challenging for a single shared policy at extreme $w_c$.

\begin{table}[h]
\caption{Preference-Conditioned DQN -- Pendulum $w_c$ sweep
  ($N_{\mathrm{eval}}=100$, 2\,M steps, single model).}
\label{tab:pref_pendulum}
\centering
\small
\resizebox{\textwidth}{!}{%
\begin{tabular}{ccccccccc}
\toprule
& \multicolumn{4}{c}{Final Model} & \multicolumn{4}{c}{Best-Model Checkpoint} \\
\cmidrule(lr){2-5}\cmidrule(lr){6-9}
$w_c$ & MSI (s) & RTA (\%) & $P_3$ & $P_4$
      & MSI (s) & RTA (\%) & $P_3$ & $P_4$ \\
\midrule
0.25 & $0.1655\pm0.0032$ & 0.60 & 0.590 & 3.268 & $0.0623\pm0.0014$ & 0.00 & 0.352 & 0.797 \\
0.5  & $0.1698\pm0.0036$ & 0.46 & 0.576 & 3.563 & $0.0596\pm0.0016$ & 0.00 & 0.347 & 0.777 \\
1.0  & $0.1709\pm0.0038$ & 0.09 & 0.589 & 3.636 & $0.0747\pm0.0044$ & 0.42 & 0.475 & 2.525 \\
2.0  & $0.1915\pm0.0058$ & 0.24 & 0.609 & 3.541 & $0.1943\pm0.0098$ & 0.52 & 0.601 & 3.580 \\
4.0  & $0.2614\pm0.0104$ & 0.29 & 0.700 & 2.808 & $0.2909\pm0.0082$ & 0.54 & 0.531 & 3.180 \\
6.0  & $0.3321\pm0.0083$ & 0.70 & 0.768 & 2.564 & $0.3661\pm0.0064$ & 0.92 & 0.408 & 2.107 \\
8.0  & $0.3450\pm0.0100$ & 1.55 & 0.764 & 2.624 & $0.3790\pm0.0083$ & 1.98 & 0.445 & 2.320 \\
10.0 & $0.3560\pm0.0080$ & 1.15 & 0.749 & 2.614 & $0.3780\pm0.0122$ & 4.10 & 0.436 & 2.199 \\
12.0 & $0.3671\pm0.0077$ & 0.94 & 0.759 & 2.494 & $0.3790\pm0.0201$ & 3.95 & 0.432 & 2.165 \\
14.0 & $0.3691\pm0.0079$ & 1.65 & 0.774 & 2.402 & $0.3904\pm0.0167$ & 1.33 & 0.366 & 1.976 \\
$\bm{16.0}$ & $\bm{0.3705\pm0.0088}$ & $\bm{1.58}$ & 0.784 & 2.418
            & $\bm{0.3926\pm0.0026}$ & $\bm{0.10}$ & 0.350 & 1.798 \\
\bottomrule
\end{tabular}}
\end{table}

\begin{table}[h]
\caption{Preference-Conditioned DQN -- CartPole $w_c$ sweep
  ($N_{\mathrm{eval}}=100$, 2\,M steps, single model).}
\label{tab:pref_cartpole}
\centering
\small
\resizebox{\textwidth}{!}{%
\begin{tabular}{ccccccccccccc}
\toprule
& \multicolumn{6}{c}{Final Model} & \multicolumn{6}{c}{Best-Model Checkpoint} \\
\cmidrule(lr){2-7}\cmidrule(lr){8-13}
$w_c$ & MSI & RTA (\%) & $P_1$ & $P_2$ & $P_3$ & $P_4$
      & MSI & RTA (\%) & $P_1$ & $P_2$ & $P_3$ & $P_4$ \\
\midrule
0.25 & $0.1280\pm0.0053$ & 0.52 & 5.637 & 5.601 & 0.534 & 6.905 & $0.0683\pm0.0112$ & 1.08 & 2.060 & 3.014 & 0.353 & 3.609 \\
0.5  & $0.1292\pm0.0050$ & 0.38 & 4.917 & 5.383 & 0.504 & 6.662 & $0.0770\pm0.0109$ & 2.71 & 1.511 & 2.388 & 0.270 & 2.867 \\
1.0  & $0.1275\pm0.0066$ & 0.44 & 3.885 & 4.738 & 0.437 & 5.854 & $0.0878\pm0.0094$ & 5.99 & 1.536 & 2.524 & 0.267 & 3.060 \\
2.0  & $0.1277\pm0.0073$ & 0.36 & 3.238 & 4.478 & 0.407 & 5.541 & $0.0916\pm0.0111$ & 6.66 & 1.894 & 2.636 & 0.271 & 3.185 \\
4.0  & $0.1313\pm0.0065$ & 0.19 & 2.957 & 4.828 & 0.428 & 6.026 & $0.1061\pm0.0141$ & 6.38 & 2.213 & 2.518 & 0.260 & 3.032 \\
6.0  & $0.1494\pm0.0057$ & 0.41 & 2.645 & 4.786 & 0.421 & 5.961 & $0.2412\pm0.0260$ & 1.34 & 3.104 & 3.174 & 0.311 & 3.779 \\
8.0  & $0.1569\pm0.0080$ & 0.74 & 2.688 & 5.015 & 0.450 & 6.148 & $0.2901\pm0.0070$ & 0.00 & 2.404 & 3.467 & 0.368 & 3.874 \\
10.0 & $0.1660\pm0.0062$ & 0.48 & 2.892 & 5.584 & 0.482 & 6.832 & $0.3132\pm0.0021$ & 0.00 & 2.882 & 3.206 & 0.366 & 3.272 \\
12.0 & $0.1792\pm0.0088$ & 0.30 & 3.182 & 5.255 & 0.446 & 6.481 & $0.3139\pm0.0028$ & 0.02 & 2.421 & 2.992 & 0.350 & 2.990 \\
14.0 & $0.1887\pm0.0092$ & 0.47 & 2.978 & 4.908 & 0.412 & 5.998 & $0.3155\pm0.0028$ & 0.04 & 2.645 & 3.017 & 0.364 & 2.885 \\
$\bm{16.0}$ & $\bm{0.1892\pm0.0106}$ & $\bm{0.77}$ & 3.154 & 5.275 & 0.439 & 6.426
            & $\bm{0.3134\pm0.0041}$ & $\bm{0.09}$ & 2.562 & 2.969 & 0.366 & 2.809 \\
\bottomrule
\end{tabular}}
\end{table}

\begin{table}[h]
\caption{Preference-Conditioned DQN -- Quadrotor $w_c$ sweep
  ($N_{\mathrm{eval}}=100$, 2\,M steps, single model).}
\label{tab:pref_quadrotor}
\centering
\small
\resizebox{\textwidth}{!}{%
\begin{tabular}{ccccccccccccc}
\toprule
& \multicolumn{6}{c}{Final Model} & \multicolumn{6}{c}{Best-Model Checkpoint} \\
\cmidrule(lr){2-7}\cmidrule(lr){8-13}
$w_c$ & MSI & RTA (\%) & $P_1$ & $P_2$ & $P_3$ & $P_4$
      & MSI & RTA (\%) & $P_1$ & $P_2$ & $P_3$ & $P_4$ \\
\midrule
0.25 & $0.1115\pm0.0075$ & 0.13 & 8.029 & 1.573 & 0.368 & 6.173 & $0.1016\pm0.0032$ & 0.03 & 0.859 & 0.896 & 0.347 & 6.632 \\
0.5  & $0.1067\pm0.0077$ & 0.35 & 7.171 & 1.782 & 0.400 & 6.500 & $0.1021\pm0.0032$ & 0.02 & 0.797 & 0.884 & 0.349 & 6.635 \\
1.0  & $0.0937\pm0.0053$ & 0.46 & 5.622 & 2.150 & 0.391 & 6.003 & $0.1039\pm0.0033$ & 0.04 & 0.777 & 0.893 & 0.348 & 6.669 \\
2.0  & $0.0689\pm0.0019$ & 0.08 & 2.757 & 1.336 & 0.222 & 3.309 & $0.1065\pm0.0035$ & 0.06 & 0.855 & 0.888 & 0.328 & 6.574 \\
4.0  & $0.0645\pm0.0007$ & 0.01 & 8.756 & 0.763 & 0.157 & 2.337 & $0.1245\pm0.0045$ & 0.37 & 1.140 & 0.915 & 0.328 & 6.588 \\
6.0  & $0.0990\pm0.0062$ & 0.13 & 3.692 & 1.501 & 0.335 & 5.197 & $0.1581\pm0.0068$ & 1.47 & 1.227 & 1.028 & 0.381 & 6.903 \\
8.0  & $0.1530\pm0.0078$ & 1.04 & 3.706 & 1.868 & 0.396 & 7.268 & $0.1966\pm0.0098$ & 2.00 & 1.347 & 1.110 & 0.397 & 6.702 \\
10.0 & $0.1641\pm0.0088$ & 10.99 & 11.162 & 1.925 & 0.520 & 6.724 & $0.2271\pm0.0119$ & 1.80 & 1.401 & 1.251 & 0.430 & 6.092 \\
12.0 & $0.1493\pm0.0097$ & 11.45 & 11.773 & 1.958 & 0.536 & 6.740 & $0.2472\pm0.0091$ & 1.46 & 1.379 & 1.361 & 0.451 & 5.700 \\
14.0 & $0.1420\pm0.0096$ & 14.63 & 11.694 & 2.135 & 0.566 & 6.878 & $0.2599\pm0.0072$ & 1.23 & 1.316 & 1.446 & 0.471 & 5.510 \\
$\bm{16.0}$ & $0.1339\pm0.0087$ & 16.03 & 11.685 & 2.180 & 0.575 & 6.721
            & $\bm{0.2660\pm0.0066}$ & 1.42 & 1.401 & 1.518 & 0.488 & 5.388 \\
\bottomrule
\end{tabular}}
\end{table}

\subsection{PPO \texorpdfstring{$w_c$}{wc} Sweep (Pendulum)}

Table~\ref{tab:ppo_pendulum} reports the full PPO sweep referenced from
Section~\ref{sec:offpolicy_required}. No PPO checkpoint achieves
communication efficiency comparable to DQN ($0.396\,\mathrm{s}$ at
$w_c=8$); runs that hold $0\%$ RTA collapse to $\taumin = 0.05\,\mathrm{s}$,
while runs that explore longer intervals incur $45$--$80\%$ RTA activation.

\begin{table}[h]
\caption{PPO -- Pendulum $w_c$ sweep ($N_{\mathrm{eval}}=100$, 2\,M steps).
  Bold marks the highest MSI in each section.}
\label{tab:ppo_pendulum}
\centering
\small
\resizebox{\textwidth}{!}{%
\begin{tabular}{ccccccccc}
\toprule
& \multicolumn{4}{c}{Final Model} & \multicolumn{4}{c}{Best-Model Checkpoint} \\
\cmidrule(lr){2-5}\cmidrule(lr){6-9}
$w_c$ & MSI (s) & RTA (\%) & $P_3(\theta)$ & $P_4(\dot\theta)$
      & MSI (s) & RTA (\%) & $P_3(\theta)$ & $P_4(\dot\theta)$ \\
\midrule
0.25 & $0.0633\pm0.0028$ & 45.02 & 0.479 & 1.412 & $0.0504\pm0.0017$ & 71.24 & 0.982 & 1.191 \\
0.5  & $0.0873\pm0.0033$ & 10.08 & 0.541 & 1.289 & $0.0635\pm0.0033$ & 4.60  & 0.686 & 0.969 \\
1.0  & $0.0885\pm0.0075$ & 30.72 & 0.336 & 0.924 & $0.0771\pm0.0078$ & 29.85 & 0.564 & 1.012 \\
2.0  & $0.0500\pm0.0000$ & 0.00  & 0.169 & 1.001 & $0.0500\pm0.0000$ & 0.00  & 0.102 & 0.998 \\
4.0  & $0.0639\pm0.0006$ & 0.00  & 0.044 & 0.626 & $0.0500\pm0.0000$ & 0.00  & 0.082 & 0.993 \\
$\bm{6.0}$
     & $\bm{0.0888\pm0.0131}$ & 66.63 & 0.494 & 1.336
     & $\bm{0.0897\pm0.0115}$ & 65.37 & 0.605 & 1.156 \\
8.0  & $0.0516\pm0.0001$ & 0.00  & 0.258 & 0.487 & $0.0500\pm0.0000$ & 0.00  & 0.241 & 0.722 \\
10.0 & $0.0735\pm0.0096$ & 77.88 & 0.481 & 1.377 & $0.0695\pm0.0122$ & 80.51 & 0.494 & 1.347 \\
\bottomrule
\end{tabular}}
\end{table}

\subsection{Quadrotor3D \texorpdfstring{$w_c$}{wc} Sweeps}

Tables~\ref{tab:sac_quadrotor3d} and~\ref{tab:pref_quadrotor3d} report
the per-$w_c$ SAC sweep and the preference-conditioned SAC sweep on
Quadrotor3D referenced from Section~\ref{sec:casestudy}.

\begin{table}[h]
\caption{SAC -- Quadrotor3D $w_c$ sweep ($N_{\mathrm{eval}}=100$, 2\,M steps).
  Bold marks the best row in each section.}
\label{tab:sac_quadrotor3d}
\centering
\small
\resizebox{\textwidth}{!}{%
\begin{tabular}{ccccccccccccc}
\toprule
& \multicolumn{6}{c}{Final Model} & \multicolumn{6}{c}{Best-Model Checkpoint} \\
\cmidrule(lr){2-7}\cmidrule(lr){8-13}
$w_c$ & MSI (s) & RTA (\%) & $P_3(\varphi)$ & $P_4(p)$ & $P_1(p_x)$ & $P_6(p_z)$
      & MSI (s) & RTA (\%) & $P_3(\varphi)$ & $P_4(p)$ & $P_1(p_x)$ & $P_6(p_z)$ \\
\midrule
0.25 & $0.0435\pm0.0001$ & 0.00 & 0.046 & 0.241 & 0.255 & 0.980 & $0.0563\pm0.0002$ & 0.00 & 0.049 & 0.270 & 0.194 & 0.736 \\
1.0  & $0.0565\pm0.0001$ & 0.00 & 0.059 & 1.348 & 0.909 & 0.367 & $0.0584\pm0.0001$ & 0.00 & 0.034 & 0.458 & 0.149 & 0.245 \\
4.0  & $0.0416\pm0.0004$ & 5.50 & 0.275 & 1.579 & 2.832 & 3.307 & $0.0565\pm0.0002$ & 0.00 & 0.044 & 0.399 & 2.146 & 1.099 \\
8.0  & $0.0627\pm0.0002$ & 0.00 & 0.056 & 0.500 & 0.341 & 0.427 & $0.0856\pm0.0005$ & 0.00 & 0.111 & 5.280 & 0.347 & 1.083 \\
16.0 & $0.0707\pm0.0010$ & 0.00 & 0.051 & 1.304 & 1.698 & 1.829 & $0.0888\pm0.0003$ & 0.00 & 0.085 & 6.885 & 1.626 & 0.223 \\
24.0 & $0.0441\pm0.0010$ & 3.57 & 0.280 & 1.315 & 4.179 & 5.524 & $0.1342\pm0.0004$ & 0.00 & 0.036 & 1.410 & 0.660 & 1.473 \\
32.0 & $0.0977\pm0.0009$ & 0.00 & 0.038 & 3.184 & 0.932 & 0.361 & $0.1023\pm0.0004$ & 0.00 & 0.059 & 4.168 & 0.500 & 1.272 \\
40.0 & $0.2042\pm0.0004$ & 0.00 & 0.048 & 0.270 & 0.405 & 3.206 & $0.2019\pm0.0004$ & 0.00 & 0.051 & 0.275 & 0.376 & 2.687 \\
$\bm{48.0}$ & $0.2962\pm0.0004$ & 0.00 & 0.045 & 0.208 & 1.572 & 2.650
            & $\bm{0.3017\pm0.0005}$ & $\bm{0.00}$ & 0.044 & 0.190 & 1.185 & 4.528 \\
56.0 & $\bm{0.2975\pm0.0002}$ & $\bm{0.00}$ & 0.042 & 0.203 & 0.277 & 3.708
     & $0.3008\pm0.0004$ & 0.00 & 0.043 & 0.195 & 0.361 & 5.289 \\
64.0 & $0.2944\pm0.0001$ & 0.00 & 0.044 & 0.252 & 0.241 & 0.329 & $0.2974\pm0.0003$ & 0.00 & 0.041 & 0.203 & 0.243 & 1.222 \\
\bottomrule
\end{tabular}}
\end{table}

\begin{table}[h]
\caption{Preference-Conditioned SAC -- Quadrotor3D $w_c$ sweep
  ($N_{\mathrm{eval}}=100$, 4\,M steps, single model, discrete sampling).
  Bold marks the best row in each section.}
\label{tab:pref_quadrotor3d}
\centering
\small
\resizebox{\textwidth}{!}{%
\begin{tabular}{ccccccccccc}
\toprule
& \multicolumn{5}{c}{Final Model} & \multicolumn{5}{c}{Best-Model Checkpoint} \\
\cmidrule(lr){2-6}\cmidrule(lr){7-11}
$w_c$ & MSI (s) & RTA (\%) & $P_3(\varphi)$ & $P_4(p)$ & $P_1(p_x)$
      & MSI (s) & RTA (\%) & $P_3(\varphi)$ & $P_4(p)$ & $P_1(p_x)$ \\
\midrule
0.25 & $0.064\pm0.000$ & 0.00 & 0.040 & 0.204 & 0.753 & $0.063\pm0.000$ & 0.00 & 0.039 & 0.195 & 0.587 \\
1.0  & $0.055\pm0.000$ & 0.00 & 0.043 & 0.215 & 0.365 & $0.058\pm0.000$ & 0.00 & 0.042 & 0.201 & 0.347 \\
4.0  & $0.071\pm0.000$ & 0.00 & 0.047 & 0.222 & 0.323 & $0.075\pm0.000$ & 0.00 & 0.045 & 0.208 & 1.006 \\
8.0  & $0.094\pm0.000$ & 0.00 & 0.044 & 0.202 & 1.138 & $0.109\pm0.000$ & 0.00 & 0.047 & 0.224 & 0.557 \\
16.0 & $0.178\pm0.001$ & 0.00 & 0.069 & 0.720 & 1.018 & $0.168\pm0.000$ & 0.01 & 0.053 & 0.649 & 2.358 \\
24.0 & $0.204\pm0.001$ & 0.00 & 0.061 & 1.717 & 2.019 & $0.207\pm0.001$ & 0.00 & 0.066 & 1.316 & 2.317 \\
32.0 & $0.202\pm0.001$ & 0.00 & 0.081 & 1.190 & 2.114 & $0.200\pm0.000$ & 0.00 & 0.061 & 0.749 & 2.066 \\
40.0 & $0.227\pm0.002$ & 0.00 & 0.194 & 1.901 & 1.889 & $0.226\pm0.000$ & 0.00 & 0.173 & 1.513 & 1.031 \\
48.0 & $0.240\pm0.001$ & 0.00 & 0.113 & 1.189 & 2.379 & $0.235\pm0.002$ & 0.00 & 0.075 & 0.658 & 1.408 \\
56.0 & $0.240\pm0.000$ & 0.00 & 0.108 & 1.244 & 2.477 & $0.238\pm0.002$ & 0.00 & 0.060 & 0.617 & 1.904 \\
$\bm{64.0}$ & $\bm{0.241\pm0.000}$ & $\bm{0.00}$ & 0.100 & 1.150 & 2.464
            & $\bm{0.239\pm0.001}$ & $\bm{0.00}$ & 0.070 & 0.725 & 2.138 \\
\bottomrule
\end{tabular}}
\end{table}

\section{Quadrotor3D Ablations, Lagrangian Comparison, and Robustness}
\label{app:q3d_detail}

Tables~\ref{tab:q3d_ablations} and~\ref{tab:q3d_lagrangian} report the
Quadrotor3D baseline/ablation comparison and the Lagrangian-SAC
comparison referenced from Section~\ref{sec:casestudy}; mass-mismatch and
disturbance results are in
Tables~\ref{tab:q3d_mismatch}--\ref{tab:q3d_disturbance}.
All Q3D results are single-seed (seed~0) per
Appendix~\ref{app:multi_seed}, with $\pm$ values reflecting per-episode
variance across the $100$ evaluation episodes.

\begin{table}[h]
\caption{Quadrotor3D baselines and Ablations A \& B at $w_c = 48$
  ($N_{\mathrm{eval}}=100$, best-model checkpoint).
  RTA\,\% is $0$ by construction for all LQR-based baselines.
  \textit{Italics} = system failure (ep.\ length ${<}3\,\mathrm{s}$);
  ``--'' = norms not meaningful for failed episodes.}
\label{tab:q3d_ablations}
\centering
\small
\resizebox{\textwidth}{!}{%
\begin{tabular}{lcccccc}
\toprule
Method & MSI (s) & RTA (\%) & $P_3(\varphi)$ & $P_4(p)$ & $P_1(p_x)$ & $P_6(p_z)$ \\
\midrule
\multicolumn{7}{l}{\textit{Baselines}} \\
\quad LQR at $\taumin$ (Baseline 1)
  & $0.040\pm0.000$ & 0.00 & 0.176 & 7.230 & 0.447 & 0.123 \\
\quad \textit{LQR at $\tau_{\mathrm{match}}=0.302\,\mathrm{s}$ (Baseline 2)}
  & \textit{0.302} & \textit{0.00} & -- & -- & -- & -- \\
\quad Classical Lyapunov-STC (Baseline 3)
  & $0.040\pm0.000$ & 0.00 & 0.176 & 7.235 & 0.444 & 0.118 \\
\midrule
\multicolumn{7}{l}{\textit{Ablation A -- No RTA}} \\
\quad No-RTA
  & $0.310\pm0.000$ & 0.00 & 0.085 & 0.375 & 3.989 & 4.037 \\
\midrule
\multicolumn{7}{l}{\textit{Ablation B -- Fixed $\tau$}} \\
\quad Fixed-$\tau$ RL ($\tau = 0.302\,\mathrm{s}$)
  & $0.302\pm0.000$ & 0.00 & 0.041 & 0.370 & 0.338 & 0.178 \\
\midrule
\multicolumn{7}{l}{\textit{Full Method}} \\
\quad \textbf{RL-STC} (SAC, $w_c=48$)
  & $\bm{0.302\pm0.001}$ & $\bm{0.00}$ & $\bm{0.044}$ & $\bm{0.190}$ & 1.185 & 4.528 \\
\bottomrule
\end{tabular}}
\end{table}

\begin{table}[h]
\caption{Lagrangian-SAC vs.\ RL-STC for Quadrotor3D ($w_c=48$,
  $N_{\mathrm{eval}}=100$). RL-STC Hard Viol.\ $= 0.0$ by construction.}
\label{tab:q3d_lagrangian}
\centering
\small
\resizebox{\textwidth}{!}{%
\begin{tabular}{lcccccc}
\toprule
& \multicolumn{3}{c}{Final Model}
& \multicolumn{3}{c}{Best-Model Checkpoint} \\
\cmidrule(lr){2-4}\cmidrule(lr){5-7}
Method & MSI (s) & Pred.\ Safety (\%) & Hard Viol.\ (\%)
       & MSI (s) & Pred.\ Safety (\%) & Hard Viol.\ (\%) \\
\midrule
Lagrangian-SAC
  & $0.311\pm0.000$ & 0.01 & 0.00
  & $0.312\pm0.000$ & 0.00 & 0.00 \\
\textbf{RL-STC}
  & $\bm{0.296\pm0.000}$ & 0.00 & $\bm{0.00}$
  & $\bm{0.302\pm0.001}$ & 0.00 & $\bm{0.00}$ \\
\bottomrule
\end{tabular}}
\end{table}

\begin{table}[h]
\caption{Quadrotor3D mass-mismatch robustness ($N_{\mathrm{eval}}=100$,
  best-model checkpoint, $w_c=48$). DR = domain-randomized training
  ($\pm40\%$ mass).}
\label{tab:q3d_mismatch}
\centering
\small
\begin{tabular}{llcc}
\toprule
Method & Scale & MSI (s) & RTA (\%) \\
\midrule
\multirow{3}{*}{RL-STC}
  & $0.7\times$ & $0.302\pm0.001$ & 0.00 \\
  & Nominal     & $0.302\pm0.001$ & 0.00 \\
  & $1.3\times$ & $0.302\pm0.001$ & 0.00 \\
\midrule
\multirow{3}{*}{RL-STC + DR}
  & $0.7\times$ & $0.276\pm0.001$ & 0.00 \\
  & Nominal     & $0.276\pm0.001$ & 0.00 \\
  & $1.3\times$ & $0.276\pm0.001$ & 0.00 \\
\midrule
\multirow{3}{*}{Classical STC}
  & $0.7\times$ & $0.040\pm0.000$ & 0.00 \\
  & Nominal     & $0.040\pm0.000$ & 0.00 \\
  & $1.3\times$ & $0.040\pm0.000$ & 0.00 \\
\bottomrule
\end{tabular}
\end{table}

\begin{table}[h]
\caption{Quadrotor3D disturbance robustness ($N_{\mathrm{eval}}=100$,
  best-model checkpoint, $w_c=48$). Disturbance applied as additive
  thrust deviation ($\delta F$, N).}
\label{tab:q3d_disturbance}
\centering
\small
\begin{tabular}{lccc}
\toprule
Condition & RL-STC MSI (s) & Cl.-STC MSI (s) & RTA (\%) \\
\midrule
No disturbance & $0.302\pm0.001$ & $0.040\pm0.000$ & 0.00 \\
\midrule
\multicolumn{4}{l}{\textit{Constant}} \\
\quad $0.5\,\mathrm{N}$ & $0.302\pm0.001$ & $0.040\pm0.000$ & 0.00 \\
\quad $1.0\,\mathrm{N}$ & $0.302\pm0.001$ & $0.040\pm0.000$ & 0.00 \\
\midrule
\multicolumn{4}{l}{\textit{Periodic}} \\
\quad $0.8\,\mathrm{N}$, $1\,\mathrm{Hz}$ & $0.302\pm0.001$ & $0.040\pm0.000$ & 0.00 \\
\quad $1.5\,\mathrm{N}$, $2\,\mathrm{Hz}$ & $0.301\pm0.001$ & $0.040\pm0.000$ & 0.00 \\
\midrule
\multicolumn{4}{l}{\textit{Impulse ($p=0.05$, random sign)}} \\
\quad $1.0\,\mathrm{N}$ & $0.302\pm0.001$ & $0.040\pm0.000$ & 0.00 \\
\quad $2.0\,\mathrm{N}$ & $0.302\pm0.001$ & $0.040\pm0.000$ & 0.00 \\
\bottomrule
\end{tabular}
\end{table}

\section{Robustness Tables}
\label{app:robustness}

Tables~\ref{tab:mismatch} and~\ref{tab:disturbance} report the full
model-mismatch and disturbance-robustness results. Detailed discussion of
these results is in Appendix~\ref{app:robustness_detail}.

\begin{table}[h]
\caption{Model-mismatch robustness: mass scaled to $0.7\times$ and
  $1.3\times$ nominal ($N_{\mathrm{eval}}=100$, seed~0 for RL-STC and
  RL-STC+DR). DR = domain randomization (mass from
  $\mathrm{Uniform}[0.6,1.4]\times$ nominal at each episode reset;
  $K$, $P$, and RTA thresholds fixed at nominal). Best-model checkpoint
  reported for all variants. \textit{Italics} = system failure
  (mean episode length ${<}5\,\mathrm{s}$); ``--'' = RTA not applicable.}
\label{tab:mismatch}
\centering
\small
\resizebox{\textwidth}{!}{%
\begin{tabular}{llcccccc}
\toprule
& & \multicolumn{2}{c}{Pendulum}
  & \multicolumn{2}{c}{CartPole}
  & \multicolumn{2}{c}{Quadrotor} \\
\cmidrule(lr){3-4}\cmidrule(lr){5-6}\cmidrule(lr){7-8}
Method & Scale & MSI (s) & RTA (\%) & MSI (s) & RTA (\%) & MSI (s) & RTA (\%) \\
\midrule
RL-STC    & $0.7\times$ & $0.246\pm0.078$ & $30.48\pm24.80$ & $0.201\pm0.025$ & $13.91\pm6.09$  & $0.284\pm0.008$ & $0.15\pm0.31$ \\
RL-STC    & Nominal     & $0.396\pm0.001$ & $0.00$          & $0.317\pm0.001$ & $0.00$          & $0.290\pm0.005$ & $0.17\pm0.33$ \\
RL-STC    & $1.3\times$ & $0.349\pm0.007$ & $4.99\pm1.68$   & $0.309\pm0.006$ & $0.81\pm1.47$   & $0.291\pm0.005$ & $0.14\pm0.28$ \\
\midrule
RL-STC+DR & $0.7\times$ & $0.383\pm0.004$ & $0.00$          & $0.296\pm0.010$ & $0.16\pm0.70$   & $0.286\pm0.005$ & $1.23\pm0.87$ \\
RL-STC+DR & Nominal     & $0.389\pm0.003$ & $0.04\pm0.25$   & $0.319\pm0.002$ & $0.02\pm0.12$   & $0.290\pm0.005$ & $1.03\pm0.84$ \\
RL-STC+DR & $1.3\times$ & $0.344\pm0.041$ & $6.42\pm13.03$  & $0.315\pm0.005$ & $0.29\pm0.79$   & $0.292\pm0.005$ & $0.88\pm0.81$ \\
\midrule
Cl.-STC   & $0.7\times$ & $\mathit{0.092\pm0.011}$ & -- & $\mathit{0.176\pm0.012}$ & -- & $0.080\pm0.000$ & -- \\
Cl.-STC   & Nominal     & $0.202\pm0.000$ & --              & $0.212\pm0.001$ & --              & $0.080\pm0.001$ & -- \\
Cl.-STC   & $1.3\times$ & $0.328\pm0.003$ & --              & $0.266\pm0.001$ & --              & $0.080\pm0.000$ & -- \\
\bottomrule
\end{tabular}}
\end{table}

\begin{table}[h]
\caption{Disturbance robustness: RL-STC MSI, Classical STC MSI, and RTA
  activation under constant, periodic, and impulse disturbances
  ($N_{\mathrm{eval}}=100$, best-model checkpoint, seed~0 for RL-STC,
  $w_c=8/16/16$). $K$, $P$, and RTA thresholds fixed at nominal.
  Pendulum: torque
  ($u_{\max}=2\,\mathrm{Nm}$); CartPole: force ($F_{\max}=20\,\mathrm{N}$);
  Quadrotor: thrust deviation ($\delta F_{\max}=5\,\mathrm{N}$). Impulse
  per RL step ($p=0.05$, random sign).}
\label{tab:disturbance}
\centering
\small
\resizebox{\textwidth}{!}{%
\begin{tabular}{lcccccccccc}
\toprule
& \multicolumn{3}{c}{Pendulum}
& \multicolumn{3}{c}{CartPole}
& \multicolumn{3}{c}{Quadrotor} \\
\cmidrule(lr){2-4}\cmidrule(lr){5-7}\cmidrule(lr){8-10}
Condition
  & RL-STC & Cl.-STC & RTA (\%)
  & RL-STC & Cl.-STC & RTA (\%)
  & RL-STC & Cl.-STC & RTA (\%) \\
\cmidrule(lr){2-4}\cmidrule(lr){5-7}\cmidrule(lr){8-10}
& \multicolumn{2}{c}{MSI (s)} &
& \multicolumn{2}{c}{MSI (s)} &
& \multicolumn{2}{c}{MSI (s)} & \\
\midrule
None
  & $0.389\pm0.003$ & $0.202$ & $0.04\pm0.25$
  & $0.319\pm0.002$ & $0.212$ & $0.02\pm0.12$
  & $0.290\pm0.005$ & $0.080$ & $1.03\pm0.84$ \\
\midrule
Const.\ 0.2\,Nm / 1.0\,N / 0.5\,N
  & $0.249\pm0.036$ & $0.200$ & $26.45\pm12.90$
  & $0.284\pm0.025$ & $0.123$ & $0.75\pm1.89$
  & $0.291\pm0.005$ & $0.082$ & $0.81\pm0.66$ \\
Const.\ 0.5\,Nm / 2.0\,N / 1.0\,N
  & $0.079\pm0.011$ & $0.199$ & $88.54\pm7.18$
  & $0.199\pm0.065$ & $0.121$ & $10.11\pm17.22$
  & $0.289\pm0.005$ & $0.082$ & $1.12\pm0.74$ \\
\midrule
Per.\ 0.3\,Nm 1\,Hz / 1.5\,N / 0.8\,N
  & $0.333\pm0.013$ & $0.200$ & $5.11\pm3.53$
  & $0.270\pm0.028$ & $0.188$ & $2.38\pm5.12$
  & $0.290\pm0.005$ & $0.083$ & $1.13\pm0.77$ \\
Per.\ 0.5\,Nm 2\,Hz / 2.5\,N / 1.5\,N
  & $0.108\pm0.024$ & $0.200$ & $77.21\pm10.70$
  & $0.250\pm0.037$ & $0.174$ & $5.42\pm9.68$
  & $0.289\pm0.005$ & $0.086$ & $0.98\pm0.77$ \\
\midrule
Imp.\ 0.5\,Nm / 2.0\,N / 1.0\,N
  & $0.252\pm0.063$ & $0.200$ & $38.38\pm18.59$
  & $0.305\pm0.017$ & $0.200$ & $0.90\pm3.22$
  & $0.289\pm0.005$ & $0.087$ & $1.10\pm0.83$ \\
Imp.\ 1.0\,Nm / 4.0\,N / 2.0\,N
  & $0.242\pm0.075$ & $0.199$ & $41.95\pm21.31$
  & $0.311\pm0.015$ & $0.204$ & $0.00\pm0.00$
  & $0.290\pm0.005$ & $0.087$ & $1.01\pm0.86$ \\
\bottomrule
\end{tabular}}
\end{table}

\section{Training Curves}
\label{app:training}

Figure~\ref{fig:training_curves} shows DQN training curves (MSI, RTA
activation rate, and per-step reward; see Section~\ref{sec:training}
for the per-step checkpoint rationale) for five $w_c$ values across
the three lower-dimensional environments. The MSI peak followed by
reward collapse at high $w_c$ directly motivates the best-model
checkpoint strategy discussed in Section~\ref{sec:cross_env}. SAC
Quadrotor3D training curves are in
Figure~\ref{fig:training_curves_q3d}; SAC exhibits substantially
smoother MSI and reward trajectories than DQN with no high-$w_c$
reward collapse, consistent with the algorithmic stability advantage
that motivates the continuous-action variant for the higher-dimensional
case study.

\begin{figure}[h]
  \centering
  \includegraphics[width=\textwidth]{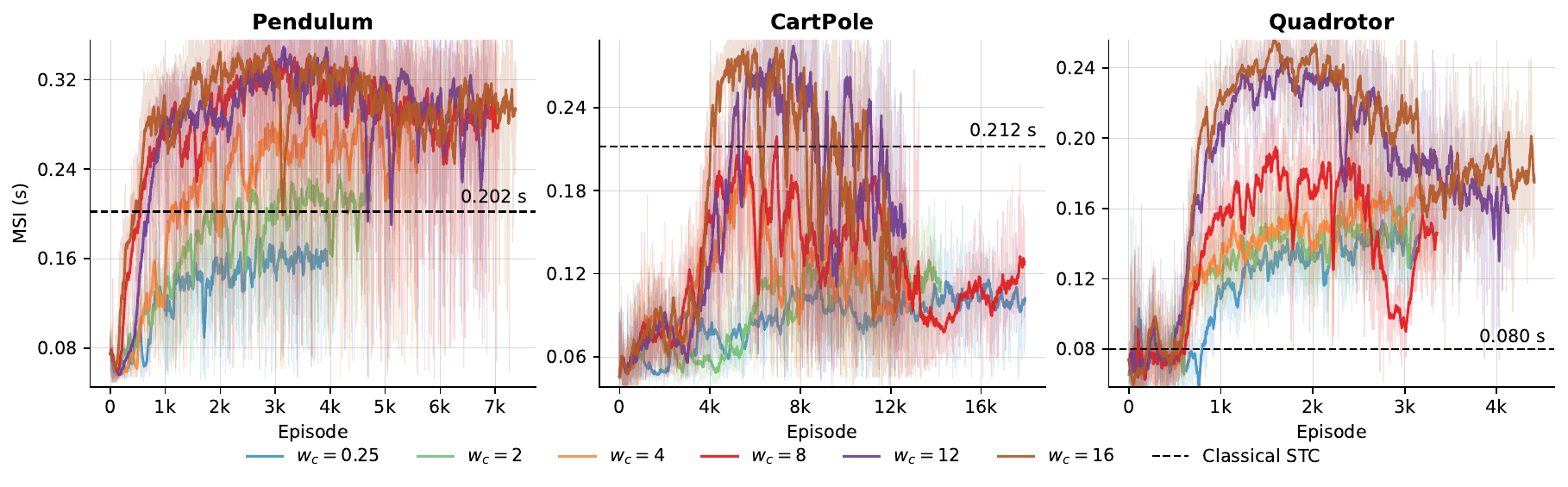}\\[4pt]
  \includegraphics[width=\textwidth]{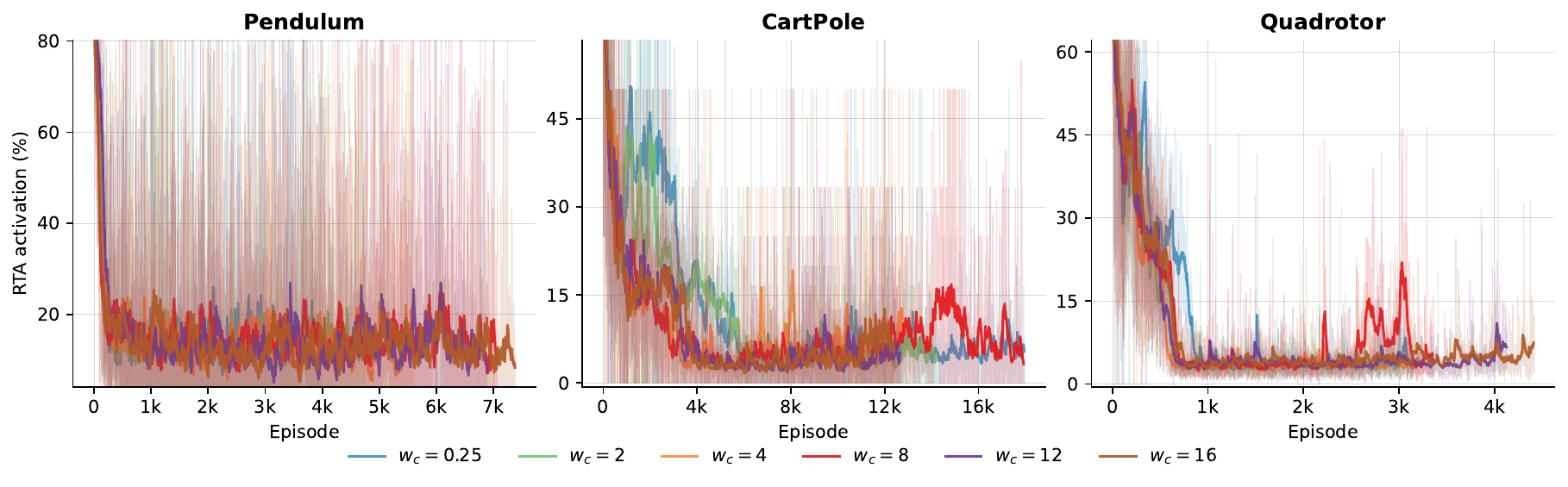}\\[4pt]
  \includegraphics[width=\textwidth]{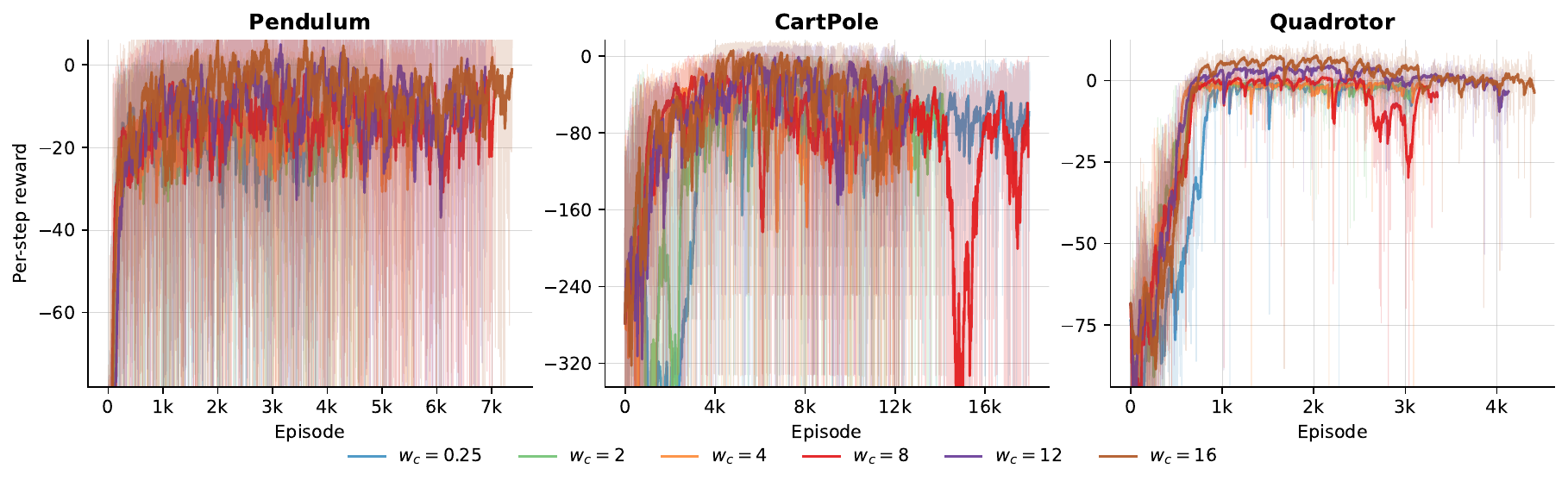}
  \caption{DQN training curves for five $w_c$ values across the three
    lower-dimensional environments (seed-0 runs from the multi-seed
    sweep). \textbf{(top)} Mean inter-sample interval; dashed line
    shows the Classical STC baseline. \textbf{(middle)} RTA activation
    rate (\%). \textbf{(bottom)} Per-step reward. Faint traces are raw
    per-episode data; bold lines are EMA-smoothed ($\alpha=0.06$,
    ${\approx}17$-episode window). Higher $w_c$ accelerates exploration of
    the sparse-sampling regime; the MSI peak before reward collapse
    motivates the best-model checkpoint strategy.}
  \label{fig:training_curves}
\end{figure}

\begin{figure}[h]
  \centering
  \includegraphics[width=\textwidth]{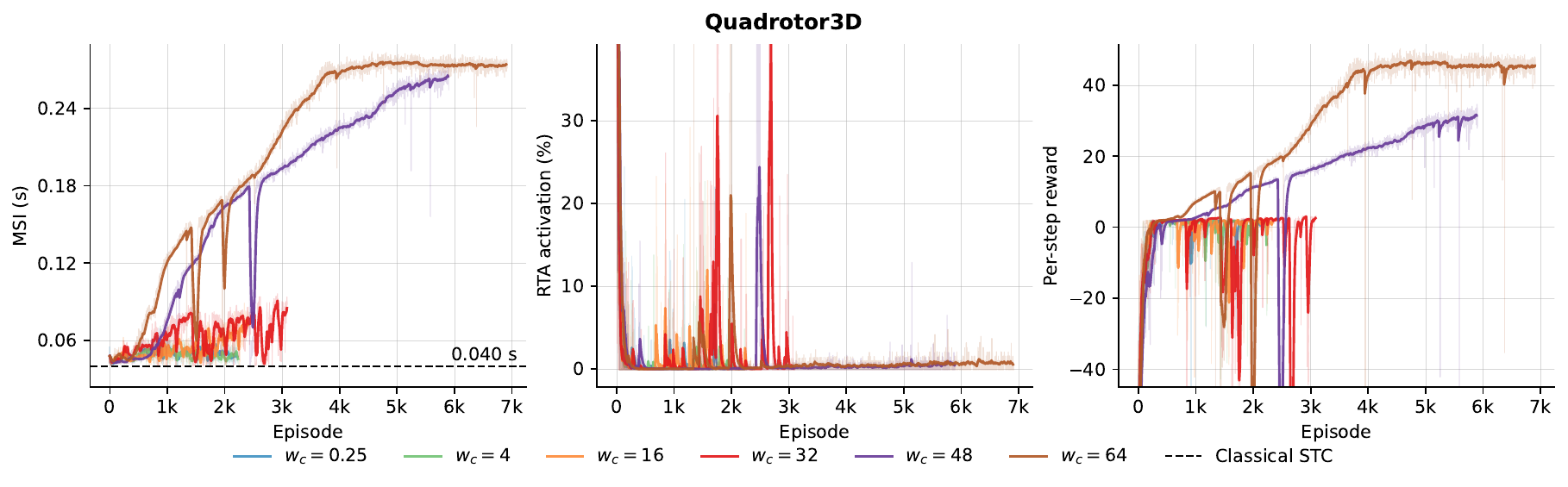}
  \caption{SAC training curves for six representative $w_c$ values on
    Quadrotor3D ($2\,\mathrm{M}$ steps, seed 0). \textbf{(left)} Mean
    inter-sample interval; dashed line shows Classical STC pinned at
    $\taumin$. \textbf{(center)} RTA activation rate (\%).
    \textbf{(right)} Per-step reward. Faint traces are raw per-episode
    data; bold lines are EMA-smoothed ($\alpha=0.06$, ${\approx}17$-episode
    window). The two-phase dynamic from Section~\ref{sec:casestudy} is
    visible: $w_c\le16$ keeps MSI near $\taumin$ during training, while
    $w_c\ge40$ saturates near $\taumax$ at $0\%$ RTA.}
  \label{fig:training_curves_q3d}
\end{figure}

\section{Detailed Per-Environment Analysis}
\label{app:per_env}

\subsection{DQN: Per-Environment Discussion}

\paragraph{Pendulum.}
The final-model MSI generally increases with $w_c$, peaking at $w_c=10$
($0.375\pm0.030\,\mathrm{s}$, $0.08\%$ RTA), with the across-seed std
reflecting training sensitivity at individual weights. The best-model
checkpoint saturates near $\taumax=0.40\,\mathrm{s}$ from $w_c=6$ onwards
with $w_c\in\{10,12,16\}$ all achieving MSI $\geq0.39\,\mathrm{s}$ and
$w_c=16$ tying $w_c=10$ at the peak ($0.397\pm0.001\,\mathrm{s}$). At the
canonical $w_c=8$ used in ablations, the best checkpoint reaches
$0.385\pm0.014\,\mathrm{s}$ ($\geq96\%$ of $\taumax$) with
$0.37\pm0.24\%$ RTA. The plateau is stable through $w_c=16$ with low RTA
across seeds ($\leq0.5\%$ for $w_c\geq6$), demonstrating that pendulum
dynamics fully saturate the sampling capacity of the $\tau$ grid at
moderate $w_c$.

\paragraph{CartPole.}
Two structural properties limit how high the final-model MSI can grow.
First, CartPole's CARE solution yields a large $\Vscale=56.6$ (versus
$8.99$ for the Pendulum), because the $\theta$ diagonal entry of $P$ is
approximately $196$, a consequence of the high LQR cost of the indirect
$x$-control path through $\theta$. Typical episode Lyapunov values are
therefore a small fraction of $\Vscale$, saturating the graded stability
term $1-\V(\x_{k+1})/\Vscale$ near $+1$ throughout training, so the
communication reward becomes the dominant differentiating signal even at
low $w_c$. Second, the tight termination angle
($\theta_{\mathrm{term}}=12^\circ$ vs.\ $60^\circ$ for the Pendulum) means
recovery from large inter-sample intervals carries higher penalty. Despite
these constraints, the best-model checkpoint reveals a clean monotonic
stability--communication tradeoff: MSI increases steadily from
$0.146\pm0.068\,\mathrm{s}$ at $w_c=0.25$ to
$0.308\pm0.014\,\mathrm{s}$ at $w_c=16$ with $0.04\%$ RTA activity,
$96\%$ of $\taumax$.

\paragraph{Quadrotor.}
The final-model trend is intermediate between Pendulum and CartPole:
$w_c=14$ yields the highest final-model MSI
($0.203\pm0.008\,\mathrm{s}$, $0.40\%$ RTA), while $w_c=16$ shows
$12.23\%$ RTA and $w_c=12$ shows $7.02\%$ RTA as the policy overshoots
the admissible inter-sample interval at the highest weights. The
best-model checkpoint improves monotonically through the sweep, with
$w_c=16$ achieving $0.281\pm0.018\,\mathrm{s}$ at $1.96\%$ RTA, $88\%$
of $\taumax$, compared to $96\%$ for Pendulum and CartPole, reflecting
the harder coupled hover dynamics. The elevated
$P_4(\dot\theta)$ values ($3.6$--$12.6$) across all weights are a structural
property of the underactuated dynamics: to correct horizontal position error
the policy must tilt the vehicle and return it to level, necessarily
generating non-zero $\dot\theta$ throughout the episode. This is not a sign
of instability but an inherent consequence of the indirect coupling between
$M$ and $x$.

\subsection{SAC: Per-Environment Discussion}

SAC reaches near-peak MSI for the Pendulum by $w_c=6$ and plateaus around
$0.33$--$0.37\,\mathrm{s}$, matching DQN's best checkpoint with zero RTA
everywhere. The DQN final model collapses at high $w_c$; SAC's final model
remains stable, suggesting the continuous policy landscape is easier to
optimize without collapsing into high-RTA regimes.

For CartPole, SAC eliminates the high-RTA collapses that plague DQN ($38.6\%$
at $w_c=0.5$, $24.2\%$ at $w_c=10$). All final models are stable and MSI
increases monotonically with $w_c$, reaching $0.227\,\mathrm{s}$ at
$w_c=16$. The best checkpoint ($0.258\,\mathrm{s}$) falls short of DQN's
multi-seed best ($0.308\pm0.014\,\mathrm{s}$), indicating that DQN
transiently discovers sparser policies but cannot sustain them; SAC
converges more reliably but has not yet matched DQN's peak at
$1\,\mathrm{M}$ steps.

For the Quadrotor, SAC matches DQN closely at moderate weights and surpasses
it at high $w_c$: the SAC best checkpoint at $w_c=16$ reaches
$0.312\,\mathrm{s}$ versus DQN's $0.281\pm0.018\,\mathrm{s}$, while
maintaining $0\%$ RTA. The final model also increases monotonically
($0.310\,\mathrm{s}$ at $w_c=16$), in contrast to DQN's $12.23\%$ RTA at
that weight.
SAC's stability advantage at high $w_c$ further validates the benefit of
continuous-action methods when DQN's discrete policy degrades.

\subsection{Preference-Conditioned DQN: Per-Environment Discussion}

For Pendulum, the preference-conditioned policy delivers a clean monotonic
stability--communication tradeoff from a single $2\,\mathrm{M}$-step
training run. The final model reaches $0.371\,\mathrm{s}$ at $w_c=16$ with
$1.6\%$ RTA and exhibits no high-$w_c$ collapses, in contrast to the
standard DQN final model. The best-model checkpoint achieves $0.393\,\mathrm{s}$,
matching the standard DQN best checkpoint ($0.397\,\mathrm{s}$) within $1\%$
while using only $\tfrac{2}{11}$ of the total training compute.

For CartPole, the preference-conditioned policy keeps RTA below $0.77\%$ in
the final model across all $w_c$ values. The best-model checkpoint plateaus
near $0.313$--$0.316\,\mathrm{s}$ above $w_c=10$, matching or slightly
exceeding the multi-seed standard DQN best ($0.308\pm0.014\,\mathrm{s}$).
A single $2\,\mathrm{M}$-step run thus produces a clean monotonic
tradeoff that matches the per-$w_c$ DQN ceiling at a fraction of the
training compute.

For the Quadrotor, the best checkpoint increases monotonically from
$0.102\,\mathrm{s}$ to $0.266\,\mathrm{s}$, within $5\%$ of multi-seed
standard DQN ($0.281\pm0.018\,\mathrm{s}$). The final model carries elevated RTA above $w_c=8$
($11$--$16\%$), indicating that the harder coupled hover dynamics remain
challenging for a single shared policy at extreme communication weights.
The best checkpoint maintains $\leq2\%$ RTA throughout the sweep, confirming
that checkpoint selection is particularly important for the quadrotor at high
$w_c$.

\section{Detailed Robustness Analysis}
\label{app:robustness_detail}

\subsection{Model Mismatch: Detailed Discussion}

\paragraph{Pendulum degrades gracefully.}
The Pendulum's safety margin is only $1.9^\circ$. At $0.7\times$ mass the
pendulum swings faster for a given torque input, so the nominal one-step-ahead
prediction systematically underestimates true angular velocity, causing RTA
activation on $30.48\pm24.80\%$ of steps. The large standard deviation
(versus $4.99\pm1.68\%$ at $1.3\times$) indicates bimodal behavior: episodes
that begin near the upright position are handled safely, while those with
larger initial angles trigger persistent RTA engagement. Crucially, the system
remains stable in all 100 episodes: the shield forces $\tauk\leftarrow\taumin$
on affected steps, cutting achieved MSI to $0.246\,\mathrm{s}$ but preventing
divergence. At $1.3\times$ mass the slower dynamics increase the margin
available to the nominal prediction, reducing RTA to $4.99\%$ and recovering
most of the MSI ($0.349\,\mathrm{s}$ vs.\ $0.396\,\mathrm{s}$ nominal).

\paragraph{Quadrotor is robust across the tested mass range.}
The Quadrotor is largely unaffected across both scaling factors (MSI within
$2.5\%$ of nominal, RTA at or below $0.2\%$), consistent with its $9.0^\circ$
safety margin absorbing the mismatch without triggering additional
interventions.

\paragraph{Baseline 2 boundary case.}
At $1.3\times$ mass, the fixed-LQR controller at $\tau_{\mathrm{match}}$
succeeds for both Pendulum (mean episode length $50.3\,\mathrm{s}$) and
CartPole ($50.2\,\mathrm{s}$), whereas both fail at nominal mass
(Table~\ref{tab:baselines_all}). The heavier plant has lower acceleration per
unit input, shifting the closed-loop eigenvalues of the ZOH discretization at
that fixed interval back into the stable region. This confirms that the nominal
failure of B2 is a genuine dynamical instability rather than a conservative
evaluation artifact, and that the RL policy's adaptive inter-sample interval
is what allows it to succeed where fixed-rate LQR cannot.

\paragraph{Domain randomization.}
Training with mass sampled from $\mathrm{Uniform}[0.6,1.4]\times$ at each
episode reset substantially reduces Pendulum sensitivity: RTA activation at
$0.7\times$ mass drops from $30.48\pm24.80\%$ (nominal training) to $0.00\%$
(DR), confirming that the agent has internalized faster dynamics through
training diversity. For CartPole, DR reduces the $0.7\times$ RTA activation
from $13.91\%$ to $0.16\%$ and recovers MSI to $0.296\,\mathrm{s}$,
essentially matching the non-DR RL-STC ($0.317\,\mathrm{s}$) while
eliminating mass sensitivity. For Quadrotor, already robust under nominal
training, DR provides negligible additional benefit. The DR models impose
minimal performance cost at the best-model checkpoint: nominal MSI values of
$0.389\,\mathrm{s}$, $0.319\,\mathrm{s}$, and $0.290\,\mathrm{s}$ match the
non-DR references closely. The benefit of DR scales with the magnitude of
perturbation: larger mass deviations increase the one-step-ahead prediction
error, causing more RTA activations and greater MSI degradation for the nominal
policy; the DR policy, trained across the full deviation range, retains
appropriate timing conservatism for those regimes and is correspondingly less
affected.

The natural resolution to mass sensitivity would be to include plant mass in
the observation, allowing the policy to condition its timing on the current
dynamics. We deliberately exclude this because our goal is to evaluate a single
fixed policy across mass variation, matching the deployment scenario where mass
is unknown. Adding mass to the observation would require online identification
of a vector quantity (for MIMO systems such as the Quadrotor, mass affects
multiple actuator channels simultaneously) and constitutes a separate problem.
The DR results therefore represent an inherent robustness--performance tradeoff:
meaningful sensitivity reduction for environments with genuine mass fragility
(Pendulum and CartPole) at negligible cost to nominal performance.

\subsection{Disturbance Robustness: Detailed Discussion}

\paragraph{RL-STC dominates Classical STC even without disturbances.}
The no-disturbance row of Table~\ref{tab:disturbance} establishes the
communication-efficiency gap before any disturbance is applied:
RL-STC achieves $0.389\,\mathrm{s}$ vs.\ $0.202\,\mathrm{s}$ for Pendulum
(+$93\%$), $0.319\,\mathrm{s}$ vs.\ $0.212\,\mathrm{s}$ for CartPole
(+$50\%$), and $0.290\,\mathrm{s}$ vs.\ $0.080\,\mathrm{s}$ for Quadrotor
(+$263\%$). Classical STC's conservative Lyapunov trigger consistently
rejects $\tau$ values the RL policy has learned to exploit safely, fixing
MSI well below the $\taumax$ ceiling in all three environments.

\paragraph{RTA acts as a graduated shield under growing disturbances.}
For RL-STC, RTA activation grows monotonically with disturbance amplitude
while every episode completes safely: the safety filter absorbs what the
policy cannot. For the Pendulum, a constant $0.5\,\mathrm{Nm}$ torque
($25\%$ of $u_{\max}$) drives RTA activation to $88.54\%$ and collapses
MSI to $\taumin=0.05\,\mathrm{s}$; at the milder $0.2\,\mathrm{Nm}$ level,
MSI falls $36\%$ (from $0.389$ to $0.249\,\mathrm{s}$) with $26\%$
activation. CartPole is substantially more robust: even the strongest
constant force tested ($2.0\,\mathrm{N}$, $10\%$ of $F_{\max}$) raises
RTA to only $10.11\pm17.22\%$ and reduces MSI by $38\%$
($0.319\to0.199\,\mathrm{s}$). The Quadrotor shows near-complete
insensitivity: MSI varies by at most $0.002\,\mathrm{s}$ and RTA stays
within $0.81$--$1.13\%$ across all conditions, because thrust disturbances
perturb the altitude channel while the angular RTA trigger monitors tilt,
leaving the timing policy structurally decoupled from the disturbance.
Classical STC's $\tau$-selection uses only the undisturbed nominal model;
disturbances enter only indirectly through the next state $\x_{k+1}$. For the
Pendulum this makes Cl.-STC MSI almost completely flat (${\approx}0.200\,\mathrm{s}$)
across all conditions, because disturbances are invisible to the trigger.
Despite maintaining higher MSI than RL-STC under strong Pendulum disturbances,
state quality is worse: under $0.5\,\mathrm{Nm}$ constant torque, Cl.-STC
achieves $P_3=0.594$ vs.\ RL-STC's $P_3=0.490$ (RL-STC's $88\%$ RTA
activation imposes frequent LQR corrections that tighten the trajectory).
Under the $0.5\,\mathrm{Nm}$ $2\,\mathrm{Hz}$ periodic case, Cl.-STC's
angular-velocity norm explodes to $P_4=3.26$ vs.\ RL-STC's $P_4=1.36$; the
Lyapunov trigger selects $\tau=0.20\,\mathrm{s}$ throughout, oblivious to the
resonant forcing.

For CartPole under constant force, Cl.-STC MSI drops significantly
($0.212\to0.123\,\mathrm{s}$ at $1.0\,\mathrm{N}$) as the drifting cart
state makes the nominal Lyapunov condition more demanding; yet despite
communicating more frequently, the cart-position norm $P_1$ worsens
dramatically ($0.065\to2.951\,\mathrm{m\cdot s}$). RL-STC simultaneously
maintains higher MSI ($0.284\,\mathrm{s}$) and better position tracking
($P_1=3.345\,\mathrm{m\cdot s}$ with near-zero RTA activation), showing the
learned policy handles constant bias more gracefully than the greedy Lyapunov
trigger. For the Quadrotor, Cl.-STC is already at $0.080\,\mathrm{s}$ in
the undisturbed case; disturbances move it only slightly ($0.080\to0.087\,\mathrm{s}$),
mirroring the structural insensitivity of RL-STC but at a communication cost
$3.6\times$ higher.

\paragraph{Impulse disturbances reveal bimodal RL-STC behavior.}
Random impulse kicks ($p=0.05$ per RL step) produce strikingly different
responses. The Pendulum is highly sensitive: a $0.5\,\mathrm{Nm}$ kick
collapses RL-STC MSI to $0.252\pm0.063\,\mathrm{s}$ and raises RTA
activation to $38.38\pm18.59\%$. The large standard deviation reflects
bimodal episode behavior: some episodes receive well-timed kicks that sustain
RTA engagement, while others see none. CartPole and Quadrotor are resilient:
the largest impulses produce $\leq1\%$ RTA activation and $\leq2\%$ MSI
change, for the same structural reasons described above (angular RTA trigger
is decoupled from the thrust/force disturbance channel for the Quadrotor).
Classical STC is insensitive to impulses in all environments, since a rare
kick shifts the state only marginally relative to the Lyapunov basin.

\begin{figure}[h]
  \centering
  \subfloat[Pendulum]{%
    \includegraphics[width=\textwidth]{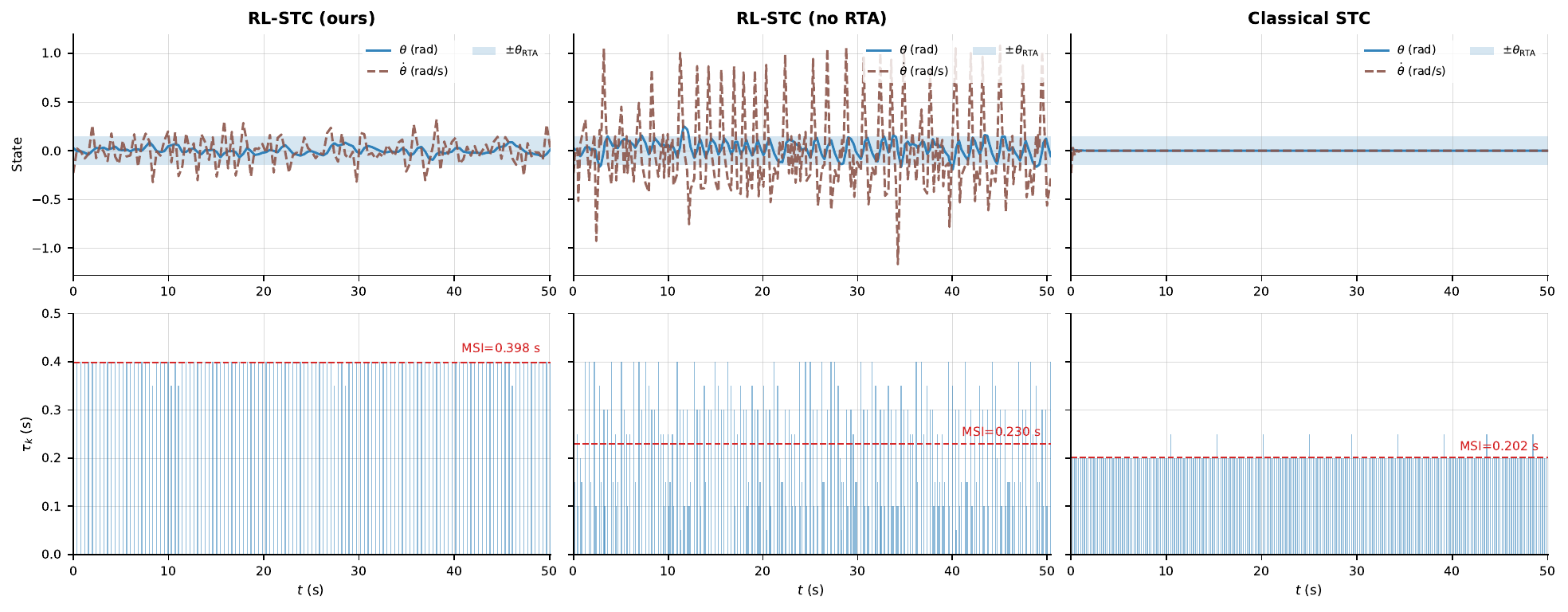}}\\[4pt]
  \subfloat[CartPole]{%
    \includegraphics[width=\textwidth]{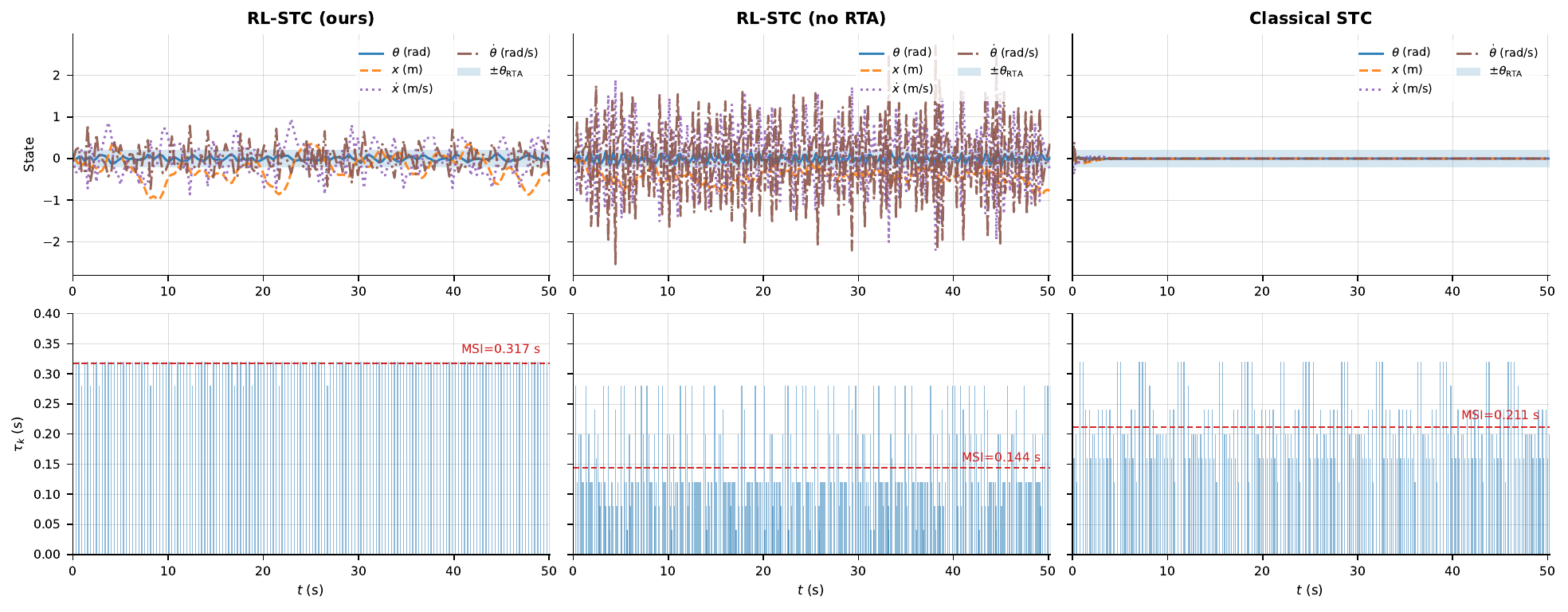}}\\[4pt]
  \subfloat[Quadrotor]{%
    \includegraphics[width=\textwidth]{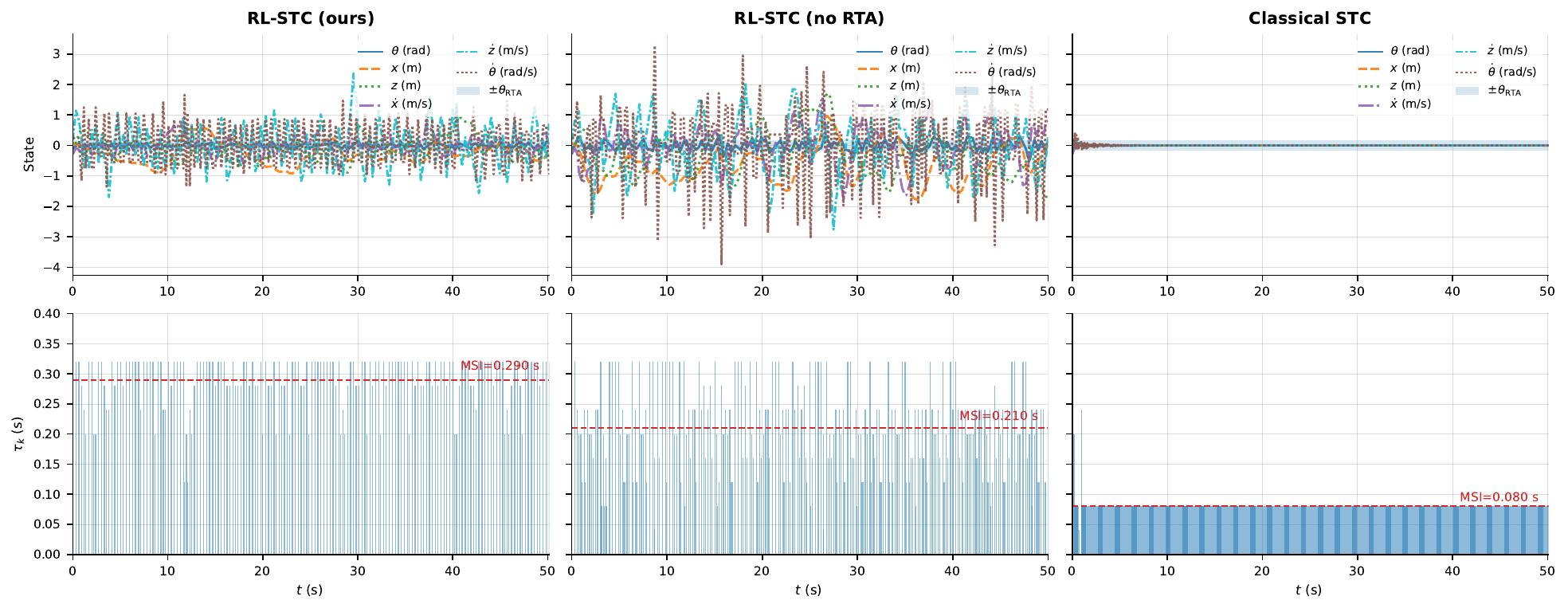}}
  \caption{Representative episode trajectories on the three
    lower-dimensional environments (seed 0, up to 1200 steps). Each
    sub-figure: \textbf{left} RL-STC (ours); \textbf{center} RL-STC
    without RTA (Ablation A); \textbf{right} Classical Lyapunov-STC.
    Top row: state trajectories with $\pm\theta_{\RTA}$ band. Bottom
    row: inter-sample intervals $\tau_k$; red stems = RTA-active steps;
    dashed line = MSI. The Quadrotor sub-figure highlights the
    dramatic MSI gap: the RL policy sustains long inter-sample
    intervals near hover while the classical trigger remains near
    $\taumin$ throughout.}
  \label{fig:trajectories}
\end{figure}

\clearpage

\end{document}